\documentclass[conference]{IEEEtran}
\IEEEoverridecommandlockouts
\usepackage{cite}
\usepackage{amsmath,amssymb,amsfonts,amsthm}
\usepackage{algorithmic}
\usepackage[ruled,vlined]{algorithm2e}
\usepackage{graphicx}
\usepackage{textcomp}
\usepackage{caption}
\usepackage{subcaption}
\usepackage{multirow}
\usepackage{xcolor}
\def\BibTeX{{\rm B\kern-.05em{\sc i\kern-.025em b}\kern-.08em
		T\kern-.1667em\lower.7ex\hbox{E}\kern-.125emX}}

\newtheorem{proposition}{Proposition}
\newtheorem{theorem}{Theorem}

\begin{document}

\title{Subspace Clustering with Active Learning\\
}

\author{\IEEEauthorblockN{Hankui Peng}
	\IEEEauthorblockA{\textit{STOR-i Centre for Doctoral Training} \\
		\textit{Lancaster University}\\
		Lancaster, UK \\
		h.peng3@lancaster.ac.uk}
	\and
	\IEEEauthorblockN{Nicos G. Pavlidis}
	\IEEEauthorblockA{\textit{Department of Management Science} \\
	\textit{Lancaster University}\\
	Lancaster, UK \\
	n.pavlidis@lancaster.ac.uk}
}

\maketitle

\begin{abstract}

Subspace clustering is a growing field of unsupervised learning that has gained
much popularity in the computer vision community. Applications can be found in areas such as motion
segmentation and face clustering.  It assumes that data originate from a union
of subspaces, and clusters the data depending on the corresponding subspace. In
practice, it is reasonable to assume that a limited amount of labels can be
obtained, potentially at a cost. Therefore, algorithms that can effectively and efficiently
incorporate this information to improve the clustering model are desirable.
In this paper, we propose an active learning framework for subspace clustering
that sequentially queries informative points and updates the subspace
model. The query stage of the proposed framework
relies on results from the perturbation theory of principal component analysis,
to identify influential and potentially misclassified points.
%
%
A constrained subspace clustering algorithm is proposed that monotonically decreases the
objective function subject to the constraints imposed by the labelled data.
We show that our proposed framework is suitable for subspace clustering
algorithms including iterative methods and spectral methods. Experiments on
synthetic data sets, motion segmentation data sets, and Yale Faces data sets
demonstrate the advantage of our proposed active strategy over
state-of-the-art.




\end{abstract}

\begin{IEEEkeywords}
	high dimensionality; active learning; subspace clustering; constrained clustering 
\end{IEEEkeywords}

\section{Introduction}

In recent years crowdsourcing \cite{su2012crowdsourcing} for data annotation
has drawn much attention in the computer vision community, due to the need to
make use of as much data as possible and the lack of sufficient
labelled data. Clustering is commonly used as an
initial step to provide a coarse preliminary grouping in the absence of labelled
data. 
%
%
For example, there are
plant recognition apps that allow one to take a photo of a plant and identify
its species. In video surveillance, one may wish to identify the points in a
sequence of frames, be it people or cars etc. Usually
some form of external information is available in these applications, either
through crowdsourcing websites, or through paid manual work
to conduct a limited amount of labelling. In either case, obtaining
labels involves a cost which is either in time, money or both. Therefore, effective and
efficient ways of carrying out data annotation are desirable.  


The process of iteratively annotating the potentially misclassified data and
subsequently updating the model is generally known as active learning
\cite{settles2008curious}. It is a subfield of machine learning that aims to
improve both supervised and unsupervised algorithms. In supervised learning,
points that are near the decision boundary are likely to be misclassified. In
unsupervised learning, the notion of potentially misclassified points is less
clear and is open to interpretation. 


In subspace clustering, points are clustered according to their underlying
subspaces. 
%
There are different ways of measuring how likely a point is
misclassified. 
One approach is to consider points whose projection onto the associated subspace is large  
as potentially misclassified
\cite{lipor2015margin}. 
Alternatively, points that are almost equidistant to their two nearest subspaces are likely to
be misclassified \cite{lipor2017leveraging}. Such ideas are based on the notion
of {\em reconstruction error} between the original point and its projection 
to the subspace that defines the cluster.
The total reconstruction error is the objective function of the $K$-Subspace Clustering (KSC) algorithm
\cite{bradley2000k}. 
We therefore argue that effective active learning strategies should explicitly
associate the query procedure with the optimisation of this objective.
%
%
However as we will discuss in greater detail in the next section,
the points with the largest reconstruction error are not 
necessarily the most informative from the perspective of
updating the entire subspace clustering model.





Motivated by 
%
%
the connection between the reconstruction error and the KSC objective, we
%
%
consider a point to be influential if querying its true class can lead to a
%
%
large decrease in the total reconstruction error.
Given a set of cluster assignments, the optimal linear subspace for
each cluster can be trivially estimated through Principal Component Analysis (PCA)
\cite{jolliffe2011principal}. In particular, the basis for each subspace (cluster)
can be defined through the set of eigenvectors of the
covariance matrix of the points that are assigned to this cluster. 
We make use of ideas 
from perturbation analysis of PCA~\cite{critchley1985influence} 
to evaluate efficiently how influential each point is, and query 
the class of the most informative point(s).
%
%
Once the true classes of the influential points have been identified, 
%
%
%
our proposed $K$-subspace clustering with constraints (KSCC) algorithm monotonically reduces
the reconstruction error while satisfying all the constraints imposed by the labelled data.
The active learning process iterates between these two query and update
procedures until the query budget is exhausted.


The rest of the paper is organised as follows. We review related work in active
learning in Section \ref{rw2}, and introduce our proposed active framework in
Section \ref{alf}. Experimental results on synthetic and real
data are discussed in Section \ref{al_er}. The paper finishes
in Section \ref{al_conclusion} with conclusions and directions for future work.

\section{Related Work}\label{rw2}


There are three main approaches to active learning \cite{settles2008curious}:
uncertainty sampling \cite{balcan2007margin}, query by committee
\cite{seung1992query}, and expected model change \cite{settles2008multiple}. 
%
%

\emph{Uncertainty sampling} queries the points the learning
algorithm is least confident about. Classic uncertainty sampling methods are generally
ignorant to the data distribution, thus prone to select outliers
\cite{donmez2007dual}. It is suggested in \cite{melville2004diverse} to measure the
informativeness of each point by the probability margin between the label
it is assigned to and its second most likely label. Other versions of
uncertainty sampling have been proposed to balance the density of a region and
the uncertainty in that region \cite{nguyen2004active}. 
When building supervised models,
one may also choose the unlabelled points near the decision
boundary.


\emph{Query by committee (QBC)} is a type of active learning strategy designed for
classifier ensembles \cite{seung1992query}. It constructs a committee of models
based on the labelled training data, and chooses to query the unlabelled points upon
which the predictions of the classifiers in the ensemble disagree the most. It enables the
training of accurate classifiers using a small subset of the data. To use this
strategy, one has to provide both the type of classifier and a measure
of disagreement among the classifiers. It has been shown in
\cite{freund1997selective} that rapid decrease in the misclassification error is
guaranteed if the queries have high expected information gain.

\emph{Expected model change} is an active learning framework that bases its query strategy on the idea that a point is informative if knowing its true class can cause a big change in the
current model \cite{settles2008multiple}. This is mostly applied to
discriminative probabilistic modelling, in which the gradient of the model is used as an indicator for the informativeness of a point.
%
%
It is widely applied to image retrieval and
text classification \cite{roy2001toward}.
The method we propose also adopts this approach, but we are the first to consider
updating an unsupervised learning model describing all the data rather than just the labelled data.

As is implied above, most active learning approaches have been
developed for supervised learning. However less attention has been paid to
the unsupervised counterpart. To the best of our knowledge, only a few active learning strategies have been proposed for subspace clustering
\cite{lipor2015margin,lipor2017leveraging}. In \cite{lipor2015margin}, two
active strategies \emph{MaxResid} and \emph{MinMargin} for KSC are proposed. 
\emph{MaxResid} queries points that have large reconstruction
error to their allocated subspaces. \emph{MinMargin} queries points that
are maximally equidistant to their two closest subspaces. 
%
These two strategies are effective in identifying the points that are most likely to be misclassified. 
However, these points are not necessarily the most informative points in terms of updating the full clustering model.

\section{Active Learning Framework} \label{alf}
In this section, we first formulate the subspace clustering problem 
and then present the proposed active learning framework.
There are two iterative procedures within this framework. The first is to identify
the most influential and potentially misclassified points. 
The second is to update the cluster assignment for all the data 
given the labelled points.

\subsection{$K$-Subspace Clustering}
A $q$-dimensional linear subspace $\mathcal{S}_{k}$, $k\in\left\{1,\ldots,K
\right\}$, can be defined through an orthonormal matrix $V_{k}\in\mathbb{R}^{P\times q}$ as
%
\begin{equation}
\mathcal{S}_{k} = \left\{\textbf{x}\in\mathbb{R}^{P}:\textbf{x}=V_{k}\textbf{y} \right\},
\end{equation}
where the columns of $V_{k}$ constitute a basis for $\mathcal{S}_{k}$.

In subspace clustering, the overall objective is to find the set of optimal cluster assignments for all data points such that the total reconstruction error between each data point to their corresponding subspaces is minimised. The loss function value $L(\textbf{x}_{i},V_{k_{i}})$ and the objective $f(\mathcal{X},\mathcal{V})$ can be written as
\begin{equation}
L(\textbf{x}_{i},V_{k_{i}}) =\| \textbf{x}_{i} - V_{k_{i}} V_{k_{i}}^\top \textbf{x}_{i} \|_{2}^{2}, \text{ and}
\end{equation}
\begin{equation}
f(\mathcal{X},\mathcal{V})=\min_{ V_{1},\ldots,V_{K} } \sum_{i=1}^{N} \min_{k_{i} \in \{1,\ldots,K\}} L(\textbf{x}_{i},V_{k_{i}}),
\label{eq3}
\end{equation}
where $\mathcal{X}=\left\{\textbf{x}_{1},\ldots, \textbf{x}_{N} \right\}$ contains the set of all $N$ points, and $\mathcal{V}=\left\{V_{1},\ldots,V_{K} \right\}$ represents the set of all subspace bases. This objective can be minimised through a $K$-means-like iterative algorithm by alternating between \emph{subspace estimation} and \emph{cluster assignment} \cite{bradley2000k}. 

Given a set of cluster assignments $\Omega=\left\{k_{1},\ldots,k_{N} \right\}$, we need to obtain the set of subspace bases such that the total reconstruction error in Eq. \eqref{eq3} is minimised. The basis matrix $V_{k}$ for each subspace $k$ can be obtained through the eigen-decomposition of its covariance matrix as 
\begin{equation}\label{eigdecomp}
(X_{k} - \boldsymbol{\mu}_k)^{T} (X_{k} -\boldsymbol{\mu}_k) = V_k^{*} \Lambda_k^{*} (V_k^{*})^{T}.
\end{equation}
We denote $X_{k}\in\mathbb{R}^{n_{k}\times P}$ as the data matrix that contains the $n_{k}$ points assigned to cluster $k$, and $\boldsymbol{\mu}_k$ as the column-wise mean vector of $X_k$. $V_k^{*}$ is a $P\times P$ matrix whose columns correspond to the eigenvectors of the covariance matrix of $X_{k}$, and $\Lambda_{k}^{*}$ is a diagonal matrix containing the $P$ eigenvalues. We denote $V_{k}$ as the subset of eigenvectors in $V_{k}^{*}$ that corresponds to the $q$ largest eigenvalues.

Given the subspace bases $\mathcal{V}=\left\{V_{1},\ldots,V_{K} \right\}$, the cluster assignment $k_{i}$ for each point $\textbf{x}_{i}\in\mathcal{X}$ can be obtained as
\begin{equation}
k_{i}=\arg\min_{k\in\left\{1,\ldots,K\right\}}\left\|\textbf{x}_{i}-V_{k}V_{k}^{T}\textbf{x}_{i} \right\|_{2}^{2}.
\end{equation}

The algorithm terminates when the loss function value in Eq.~\eqref{eq3} stops decreasing, which means either a local or global optimum is reached.

\subsection{Query Procedure} \label{query}

The first element in quantifying the influence of an unlabelled point is 
the reduction in the reconstruction error that would be achieved
if this point is removed from its currently assigned cluster. It is important
to note that removing a point from a cluster implies that the basis for the associated
linear subspace can change, because $V_{k}$ is a function of $X_k$ (see Eq.~\eqref{eigdecomp}).
Explicitly, we define $U_1(\textbf{x}_{s},V_{k_{s}})$ as the decrease
in the reconstruction error after removing the queried point $\textbf{x}_s$ from cluster $k_s$ as follows
\begin{equation}\label{eq_u1}
U_{1}(\textbf{x}_{s},V_{k_{s}})=\sum_{\textbf{x}\in\mathcal{X}_{k_{s}}}L(\textbf{x},V_{k_{s}})-\sum_{\textbf{x}\in\mathcal{X}_{k_{s}}\backslash\left\{\textbf{x}_{s}\right\}}L(\textbf{x},\tilde{V}_{k_{s}}),
\end{equation}
%
%
where $\mathcal{X}_{k_s}$ denotes the set of points in cluster $k_s$, and $V_{k_{s}}$
denotes the basis for cluster $k_{s}$. We use $\tilde{V}_{k_{s}}$ to denote the potentially perturbed basis after point $\textbf{x}_s$ is removed.

The second element in quantifying the influence of an unlabelled point is to consider
the increase in the reconstruction error of the cluster that $\textbf{x}_{s}$ will be assigned to (after
being removed from its current cluster $k_{s}$). As before, adding a point to a cluster implies
that the associated basis for this cluster can change.
%
%
%
%
%
%
Given that points are allocated to their closest subspace, it is thus sensible to assume
the cluster that $\textbf{x}_{s}$ has the second smallest reconstruction error to is where $\textbf{x}_{s}$ would be assigned next. This can be expressed as 
%
%

\begin{equation}
k^{*}_{s}=\arg\min_{k\in\left\{1,\ldots,K \right\}\backslash \{k_{s}\}}L(\textbf{x}_{s},V_{k}).
\end{equation} 
Then we can define $U_{2}(\textbf{x}_{s},V_{k_{s}^{*}})$ as the increase in the reconstruction error after adding $\textbf{x}_{s}$ to cluster $k_{s}^{*}$, which can be expressed as 
\begin{equation}\label{eq_u2}
U_{2}(\textbf{x}_{s},V_{k_{s}^{*}})=\sum_{\textbf{x}\in\mathcal{X}_{k_{s}^{*}}\cup\left\{\textbf{x}_{s}\right\}}L(\textbf{x},\tilde{V}_{k_{s}^{*}})-\sum_{\textbf{x}\in\mathcal{X}_{k_{s}^{*}}}L(\textbf{x},V_{k_{s}^{*}}).
\end{equation}
Here $\tilde{V}_{k_{s}^{*}}$ is the $P\times q$ basis matrix whose columns are the
eigenvectors of the covariance matrix of the points in the set $\left\{\mathcal{X}_{k_{s}^{*}}\cup\left\{\textbf{x}_{s}\right\}\right\}$.




Combining the above two measures of influence together, we determine the most informative and likely to be misclassified point $\textbf{x}_{s}^{*}$ as
\begin{equation} \label{eq_al_obj}
\textbf{x}_{s}^{*}=\arg\max_{\textbf{x}_{s}\in\mathcal{X}_{U}} \left\{U_{1}(\textbf{x}_{s},V_{k_{s}})-U_{2}(\textbf{x}_{s},V_{k_{s}^{*}}) \right\},
\end{equation}
where $\mathcal{X}_{U}$ denotes the set of unlabelled points, and $\mathcal{X}_{L}$ the set of labelled points. Eq. \eqref{eq_al_obj} gives the point that brings the largest decrease in the reconstruction error once removed from its allocated cluster $k_{s}$, and the smallest increase in reconstruction error upon being reallocated to its most probable cluster $k_{s}^{*}$. 
%
%
%



Although these two measures of influence can be quantified and calculated exactly, the number of required SVD computations is $\mathcal{O}(N^{2})$ throughout all iterations. 
Not to mention that the computational complexity of SVD is $\text{min}\left\{N^{2}P,P^{2}N \right\}$~\cite{golub2012matrix}.
Every time a point is removed from or added to a cluster, the subspace bases change and need to be recalculated through PCA. Hence the need to seek for an alternative approach, which could be pursued through the perturbation analysis of PCA \cite{critchley1985influence}. 


We approximate the perturbed covariance matrix, the perturbed eigenvectors and eigenvalues through power series expansions \cite{shi1997local}. As such, we can obtain expressions for the updated reconstruction error without having to recompute all the updated eigenvalues and eigenvectors after data deletion or addition. The algorithmic form of the query strategy is provided in Algorithm \ref{algo_al} before we detail the specifics of how the two influence measures are calculated.

\begin{algorithm}[h!]
	\caption{Query Strategy}
	\DontPrintSemicolon
	\SetAlgoLined
	\SetKwInOut{Input}{Input}
	\SetKwInOut{Output}{Output}
	\Input{Data matrix $X \in \mathbb{R}^{N\times  P}$ \\
		Number of clusters~$K$ \\
		Initial cluster assignment $\left\{k_{1},\ldots,k_{N} \right\}$}
	\BlankLine
	\Repeat{\text Budget $T$ or desired performance is reached}{
		\For{$\textbf{x}_{s}\in\mathcal{X}_{U}$}{
			Compute the influence $U_{1}(\textbf{x}_{s}, V_{k_{s}})$ of removing $\textbf{x}_{s}$ from its allocated cluster $k_{s}$ \;
			
			
			Calculate $k^{*}_{s}=\arg\min_{k\in\left\{1,\ldots,K \right\}\backslash\{ k_{s}\}}L(\textbf{x}_{s},V_{k})$\;
			Calculate $U_{2}(\textbf{x}_{s},V_{k_{s}^{*}})$ using Eq. \eqref{eq_u2}
		}
		Optimise Eq. \eqref{eq_al_obj} to query $\textbf{x}_{s}^{*}$ and its true class $l_{s}\in\left\{1, \ldots, K\right\}$
		\BlankLine
	}

	\label{algo_al}
\end{algorithm}
\textbf{The influence of data deletion.} 
Let $S$ denote the sample covariance matrix for the points that belong to the same cluster, $\lambda_{1}\ldots,\lambda_{P}$ denote its eigenvalues in descending order, and $\textbf{v}_{1},\ldots, \textbf{v}_{P}$ denote its eigenvectors. Let $I$ denote a set of $l$ points to be removed from the cluster. As a result, the covariance matrix, its eigenvectors and eigenvalues will be perturbed by a certain amount. Under small perturbations ($0<\epsilon<1$), the perturbed covariance matrix $S(\epsilon)$, the $k$th perturbed eigenvalue $\lambda_{k}(\epsilon)$ and the $k$th perturbed eigenvector $\textbf{v}_{k
}(\epsilon)$ can be written as the following convergent power series
\begin{equation} 
\label{eq_power}
\begin{aligned}
&S(\epsilon)=S+S^{(1)}\epsilon+S^{(2)}\epsilon^{2}+\cdots+S^{(m)}\epsilon^{m}+\cdots,\\
&\lambda_{k}(\epsilon)=\lambda_{k}+\alpha_{1}\epsilon+\alpha_{2}\epsilon^{2}+\cdots+\alpha_{m}\epsilon^{m}+\cdots, \\
&\textbf{v}_{k}(\epsilon)=\textbf{v}_{k}+\boldsymbol{\psi}_{1}\epsilon+\boldsymbol{\psi}_{2}\epsilon^{2}+\cdots+\boldsymbol{\psi}_{m}\epsilon^{m}+\cdots.
\end{aligned}
\end{equation} 
For sufficiently small $\epsilon$, the order of the eigenvalues is maintained, so are the signs of the eigenvectors \cite{enguix2005influence}.

The main interest lies in finding the coefficients in the power series approximations. First the perturbed sample covariance matrix $S_{(I)}^{-}$ can be deduced from basic definitions of the covariance matrix \cite{wang1993effects, benasseni2018correction},
\begin{equation}
\begin{split}
S_{(I)}^{-}=S+\frac{l}{n-l}\left[(S-S_{I})-(\bar{\textbf{x}}_{I}-\bar{\textbf{x}})(\bar{\textbf{x}}_{I}-\bar{\textbf{x}})^{T} \right]
\\-\frac{l^{2}}{(n-l)^{2}}(\bar{\textbf{x}}_{I}-\bar{\textbf{x}})(\bar{\textbf{x}}_{I}-\bar{\textbf{x}})^{T}.
\end{split}
\end{equation}
In the above expression, $n$ denotes the original number of points in the cluster. We use $\bar{\textbf{x}} \in\mathbb{R}^{P}$ to denote the feature-wise mean vector of the data before the removal of $l$ points, and $\bar{\textbf{x}}_{I}$ the feature-wise mean vector of the $l$ points to be removed. Lastly, we use $S$, $S_{I}$, and $S_{(I)}^{-}$ to denote the original covariance matrix, the covariance matrix of the deleted data, and the covariance matrix of the perturbed data, respectively. We can associate $\frac{l}{n-l}$ and $\left((S-S_{I})-(\bar{\textbf{x}}_{I}-\bar{\textbf{x}})(\bar{\textbf{x}}_{I}-\bar{\textbf{x}})^{T} \right)$ with $\epsilon$ and $S^{(1)}$ as in Eq. \eqref{eq_power}. Similarly, the correspondence can be made for the second order coefficients. We use a first order approximation for our purpose from now on, as it has been shown to be sufficiently accurate \cite{wang1993effects}.

As for the coefficients in the approximations for the eigenvalues and eigenvectors, Lemma 2 in \cite{wang1993effects} provides us with the following results 
\begin{equation}
\alpha_{1} = \textbf{v}_{k}^{T}S^{(1)}\textbf{v}_{k}, \text{   } \alpha_{m}=\textbf{v}_{k}^{T}S^{(1)}\boldsymbol{\psi}_{m-1};
\label{eq_pert}
\end{equation}
and
\begin{equation}
\begin{split}
&\boldsymbol{\psi}_{1}=-(S-\lambda_{k} \textbf{I})^{+}S^{(1)}\textbf{v}_{k},\\
&\boldsymbol{\psi}_{m}=-(S-\lambda_{k} \textbf{I})^{+}\left(S^{(1)}\boldsymbol{\psi}_{m-1}-\sum_{i=1}^{m-1}\alpha_{i}\boldsymbol{\psi}_{m-i} \right).
\end{split}
\end{equation} 
In the above expression, we have the Moore-Penrose inverse $(S-\lambda_{k} I)^{+}=\sum_{j\neq k}\frac{\textbf{v}_{j}\textbf{v}_{j}^{T}}{(\lambda_{j}-\lambda_{k})}$ \cite{golub2012matrix}. Based on the above, we can deduce expressions for the perturbed eigenvalues, and the influence of data deletion as expressed in Eq. \eqref{eq_u1}. 

We start with writing the first order approximation of the $k$th ($k\in\left\{1,\ldots,P \right\}$) perturbed eigenvalue as follows
\begin{equation}
\begin{aligned}
\lambda_{k}(\epsilon)&=\lambda_{k}+\epsilon\lambda_{k}^{(1)}+\mathcal{O}(\epsilon^{2})\\
&\approx\lambda_{k}+\epsilon \textbf{v}_{k}^{T}S^{(1)}\textbf{v}_{k}\\
&= \lambda_{k}+\frac{l}{n-l} \textbf{v}_{k}^{T}\left[S-\frac{1}{l}\sum_{i\in I} (\textbf{x}_{i}-\bar{\textbf{x}})(\textbf{x}_{i}-\bar{\textbf{x}})^{T}\right]\textbf{v}_{k}\\
&=\lambda_{k}+\frac{l}{n-l}\left[\lambda_{k}-\frac{1}{l}\sum_{i\in I}\alpha_{ki}^{2} \right]\\
&=\frac{n}{n-l}\lambda_{k}-\frac{1}{n-l}\sum_{i\in I}\alpha_{ki}^{2},
\end{aligned}
\end{equation}
where $\alpha_{ki}=\textbf{v}_{k}^{T}(\textbf{x}_{i}-\bar{\textbf{x}})$. Then using the expression in Eq.~\eqref{eq_u1}, we can write the influence of removing a set $I$ of data~$X_{I}\in\mathbb{R}^{l\times P}$~from cluster $k$ as
\begin{equation}
\begin{aligned}
U_{1}(X_{I},k)
&=\sum_{\textbf{x}\in\mathcal{X}_{k}}L(\textbf{x},V_{k})-\sum_{\textbf{x}\in\mathcal{X}_{k}\backslash\left\{\textbf{x}_{i}:i\in I \right\}}L(\textbf{x},\tilde{V}_{k})\\
&=\sum_{k=q+1}^{P}\lambda_{k}-\sum_{k=q+1}^{P}\lambda_{k}(\epsilon)\\
&=\sum_{k=q+1}^{P}\left(\frac{1}{n-l}\sum_{i\in I}\alpha_{ki}^{2}-\frac{l}{n-l}\lambda_{k} \right).
\end{aligned}
\end{equation}
One can obtain the influence for the deletion of one point by plugging in $l=1$. The deduction follows due to the equivalence between the reconstruction error and the sum of the unused eigenvalues in representing the subspace \cite{jolliffe2011principal}.

\textbf{The influence of data addition.} In the previous section, we have shown the influence of data deletion through perturbation analysis of the eigenvalues and eigenvectors. The aim is to find influential points whose true classes might differ from their currently allocated labels. 
%
Now we assess the impact on the reconstruction error for the cluster to which the removed points are added.
Following the same line of analysis as before, we now let $X$ denote the data matrix that the set of $l$ points are to be added to, and $n$ the number of points in $X$. We denote the data after the addition of $l$ points as $X_{I_{+}}$, and the corresponding sample covariance matrix $S_{(I)}^{+}$  which combines the original data $X$ and the data to be added $X_{I}$. 

\begin{proposition}
	The form of $S_{(I)}^{+}$ can be expressed as follows,
	\begin{equation}\label{eq_si+}
	\begin{split}
	S_{(I)}^{+}=S+\frac{l}{n+l}\left[(S_{I}-S)-(\bar{\textbf{x}}_{I}+\bar{\textbf{x}})(\bar{\textbf{x}}_{I}+\bar{\textbf{x}})^{T} \right]\\
	+\frac{l^{2}}{(n+l)^{2}}\left(\bar{\textbf{x}}_{I}+\bar{\textbf{x}} \right)\left(\bar{\textbf{x}}_{I}+\bar{\textbf{x}} \right)^{T}.
	\end{split}
	\end{equation}
	\label{prop1}
\end{proposition}
The proof of Proposition \ref{prop1} can be found in the Appendix. It is easy to see that this can be matched exactly with the first two orders of the power series expansion. Interestingly, the form of the perturbed covariance matrix in the case of single data addition cannot be obtained trivially by setting $l=1$ in Eq. \eqref{eq_si+}. We show the perturbed form of the covariance matrix for the case of single data addition in Proposition \ref{prop2}, with the proof included in the Appendix.

\begin{proposition}
	\label{prop2}
	The perturbed covariance matrix in the case when $l=1$ can be expressed as
	\begin{equation}
	\begin{split}
	S_{(i)}^{+}=S+\frac{1}{n+1}\left[\left(\bar{\textbf{x}}-\textbf{x}_{i} \right)\left(\bar{\textbf{x}}-\textbf{x}_{i} \right)^{T}-S \right]\\
	-\frac{1}{(n+1)^2}\left(\bar{\textbf{x}}-\textbf{x}_{i} \right)\left(\bar{\textbf{x}}-\textbf{x}_{i} \right)^{T},
	\end{split}
	\end{equation}
	in which $\frac{1}{n+1}$ and $\left[\left(\bar{\textbf{x}}-\textbf{x}_{i} \right)\left(\bar{\textbf{x}}-\textbf{x}_{i} \right)^{T}-S\right]$ correspond to $\epsilon$ and $S^{(1)}$ respectively.
\end{proposition}

Using the above expression for the perturbed covariance matrix and the results in Eq. \eqref{eq_pert}, we express the first order approximation of the $k$th perturbed eigenvalue for $l=1$ as
\begin{equation}
\begin{aligned}
\lambda_{k}(\epsilon)&=\lambda_{k}+\epsilon\lambda_{k}^{(1)}+\mathcal{O}(\epsilon^{2})\\
&\approx\lambda_{k}+\epsilon \textbf{v}_{k}^{T}S^{(1)}\textbf{v}_{k}\\
&= \lambda_{k}+\frac{1}{n+1} \textbf{v}_{k}^{T}\left(\left(\bar{\textbf{x}}-\textbf{x}_{i} \right)\left(\bar{\textbf{x}}-\textbf{x}_{i} \right)^{T}-S \right)\textbf{v}_{k}\\
&=\frac{1}{n+1}\alpha_{ki}^2+\frac{n}{n+1} \lambda_{k},
\end{aligned}
\end{equation}  
where $\alpha_{ki}=\textbf{v}_{k}^{T}(\textbf{x}_{i}-\bar{\textbf{x}})$ as before. Hence, the change in the reconstruction error for cluster $k^{*}_{s}$ after the addition of $\textbf{x}_{s}$ can be expressed as

\begin{equation}
\begin{aligned}
U_{2}(\textbf{x}_{s},V_{k^{*}_{s}})
&=\sum_{\textbf{x}\in\mathcal{X}_{k^{*}_{s}}\cup\left\{\textbf{x}_{s} \right\}}L(\textbf{x},\tilde{V}_{k^{*}_{s}})-\sum_{\textbf{x}\in\mathcal{X}_{k^{*}_{s}}}L(\textbf{x},V_{k^{*}_{s}})\\
&=\sum_{k=q+1}^{P}\left(\lambda_{k}(\epsilon)-\lambda_{k}\right)\\
&=\sum_{k=q+1}^{P}\frac{\alpha_{ki}^{2}-\lambda_{k}}{n+1}.
\end{aligned}
\end{equation} 

Using the perturbation analysis results, the influence of data addition and deletion can be calculated directly after $K$ SVD computations per iteration. This means we only need $(T\cdot~K)$ SVD computations for all $T$ iterations as compared to $\mathcal{O}(T\cdot~N^{2})$.

\subsection{Update Procedure}
After the class memberships of some points are queried, we will know the pairwise must-link and cannot-link relationships among them. However, we do not know to which cluster labels we should assign each of these points to. 
The next step is to update the subspace model under the grouping constraints. That is, the queried points that belong to the same class must be assigned to the same cluster label. Additionally, the queried points that do not belong to the same class should be assigned different cluster labels. 

We can naturally extend KSC into an iterative constrained clustering algorithm with three stages. The first two stages involve the estimation of subspace bases and the cluster assignment of each point to the closest subspace. In the third stage, we satisfy the grouping constraints as mentioned above. This gives us a new constrained clustering objective, which we can divide into two parts.

For the set of unlabelled data $\mathcal{X}_{U}$, the subspace clustering objective is to minimise
\begin{equation}
L(\mathcal{X}_{U},\mathcal{V}) = \sum_{\textbf{x}_{u}\in\mathcal{X}_{U}}\left\{\min_{m\in\left\{1,\ldots,K \right\}}\left\|\textbf{x}_{u}-V_{m}V_{m}^{T}\textbf{x}_{u} \right\|_{2}^{2} \right\},
\label{eq_obj1}
\end{equation}
where $V_{m}$ is a $P\times q$ basis matrix that is determined by the points that are currently allocated to subspace $m$. 
Note that the basis matrix of the $m$th cluster $V_{m}$ is determined by points that are both labelled and unlabelled.

For the set of labelled data $\mathcal{X}_{L}$, we need to both minimise the reconstruction error without violating any of the grouping constraints. 
%
%
Among $K$ groups of queried points, there are $K!$ ways of matching each group to a unique cluster label. This is a combinatorial optimisation problem, and we can denote as $\mathcal{P}(K)$ the set of all possible permutations. Let $P_{n}$ be one realisation of $\mathcal{P}(K)$ that contains $K$ unique assignment labels to be matched with the queried points, and $P_{nl}$ be the assigned cluster label in the $n$th permutation that corresponds to true class $l$. 
Then we can write the subspace clustering objective for the labelled data $\mathcal{X}_{L}$ as 
\begin{equation}
L(\mathcal{X}_{L},\mathcal{V}) = \min_{\substack{P_{n}\in\mathcal{P}(K)\\n\in\left\{1,\ldots,K! \right\}}}\left\{\sum_{l=1}^{K} \left\|X_{l}-V_{P_{nl}}V_{P_{nl}}^{T}X_{l}  \right\|_{2}^{2} \right\},
\label{eq_updatel}
\end{equation}
where $X_{l}$ is a $n_{l}\times P$ matrix that contains the $n_{l}$ queried points from class $l$.

When $K$ is small, it is easy to simply evaluate all $K!$ permutations and choose the one with the smallest overall cost. However as the number of clusters grows, it is computationally prohibitive to evaluate all combinatorial possibilities. 
It is also known as the \emph{minimum weight perfect matching problem}, which can be solved in polynomial time through the Hungarian algorithm \cite{kuhn1955hungarian}. 
We first construct a $K$ by $K$ cost matrix $P$ in which the $(n,l)$-th entry $P_{nl}$ denotes the total reconstruction error of allocating data from class $n$ to cluster label $l$. In this algorithm, the computational cost is upper bounded by $\mathcal{O}(K^3)$. We adopt the Hungarian algorithm as an alternative approach to exhaustive search to our problem in stage 3 when $K!$ is larger than $K^{3}$.

To combine both the unlabelled and labelled objectives together, we can write a constrained objective as
\begin{equation}
\begin{split}
g(\mathcal{X},\mathcal{V})=\sum_{\textbf{x}_{u}\in\mathcal{X}_{U}}\left\{\min_{m\in\left\{1,\ldots,K \right\}}\left\|\textbf{x}_{u}-V_{m}V_{m}^{T}\textbf{x}_{u} \right\|_{2}^{2} \right\}\\
+\min_{\substack{P_{n}\in\mathcal{P}(K),\\n\in\left\{1,\ldots,K! \right\}}}\left\{\sum_{l=1}^{K} \left\|X_{l}-V_{P_{nl}}V_{P_{nl}}^{T}X_{l}  \right\|_{2}^{2} \right\}.
\end{split}
\label{eq_kscc_obj}
\end{equation}
The procedural form of \emph{KSC with Constraints (KSCC)} is detailed in Algorithm \ref{algo_update}. 
This three-stage procedure ensures that the constrained subspace clustering objective decreases monotonically while satisfying all grouping constraints. A detailed proof for this can be found in Theorem \ref{thm_1}.

\begin{theorem}\label{thm_1}
	The $K$-subspace clustering with constraints (KSCC) algorithm decreases the objective in Eq. \eqref{eq_kscc_obj} monotonically throughout iterations.
\end{theorem}
\begin{proof}
	Our proof borrows ideas from the proof for the monotonicity of the $K$-means algorithm. In the first iteration, we have as input a set of cluster assignment $\Omega^{(1)}=\left\{k_{1}^{(1)},\ldots,k_{N}^{(1)} \right\}$, the set of unlabelled data $\mathcal{X}_{U}$ and labelled data $\mathcal{X}_{L}$. In the first iteration, we can calculate a set of bases matrices $\mathcal{V}^{(1)}$ for all subspaces. Let $g(\mathcal{X}, \mathcal{V}^{(1)})$ be the combined reconstruction error at iteration 1, then at iteration $t$ ($t=1,\ldots,T$) we have
	%
	
	\begin{equation}
	\begin{aligned}
	g(\mathcal{X}, \mathcal{V}^{(t)})&=
	\sum_{\textbf{x}_{u}\in\mathcal{X}_{U}}\left\|\textbf{x}_{u}-V_{k_{u}^{(t)}}^{(t)}[V_{k_{u}^{(t)}}^{(t)}]^{T}\textbf{x}_{u} \right\|_{2}^{2}\\&\hspace{1.8em}+\sum_{l=1}^{K} \left\|X_{l}-V_{P_{nl}^{(t)}}^{(t)}[V_{P_{nl}^{(t)}}^{(t)}]^{T}X_{l}  \right\|_{2}^{2}\\
	&\geq \hspace{.5em}\sum_{\textbf{x}_{i}\in\mathcal{X}} \left\|\textbf{x}_{i}-V_{k_{i}^{(t)}}^{(t+1)}[V_{k_{i}^{(t)}}^{(t+1)}]^{T}\textbf{x}_{i} \right\|_{2}^{2}\\
	&=\sum_{\textbf{x}_{u}\in\mathcal{X}_{U}}\left\|\textbf{x}_{u}-V_{k_{u}^{(t)}}^{(t+1)}[V_{k_{u}^{(t)}}^{(t+1)}]^{T}\textbf{x}_{u} \right\|_{2}^{2}\\&\hspace{1.8em}+\sum_{l=1}^{K} \left\|X_{l}-V_{P_{nl}^{(t)}}^{(t+1)}[V_{P_{nl}^{(t)}}^{(t+1)}]^{T}X_{l}  \right\|_{2}^{2}\\
	&\geq\sum_{\textbf{x}_{u}\in\mathcal{X}_{U}}\left\|\textbf{x}_{u}-V_{k_{u}^{(t+1)}}^{(t+1)}[V_{k_{u}^{(t+1)}}^{(t+1)}]^{T}\textbf{x}_{u} \right\|_{2}^{2}\\&\hspace{1.8em}+\sum_{l=1}^{K} \left\|X_{l}-V_{P_{nl}^{(t)}}^{(t+1)}[V_{P_{nl}^{(t)}}^{(t+1)}]^{T}X_{l}  \right\|_{2}^{2}\\
	&=\sum_{\textbf{x}_{u}\in\mathcal{X}_{U}}\left\|\textbf{x}_{u}-V_{k_{u}^{(t+1)}}^{(t+1)}[V_{k_{u}^{(t+1)}}^{(t+1)}]^{T}\textbf{x}_{u} \right\|_{2}^{2}\\&\hspace{1.8em}+\sum_{l=1}^{K} \left\|X_{l}-V_{P_{nl}^{(t+1)}}^{(t+1)}[V_{P_{nl}^{(t+1)}}^{(t+1)}]^{T}X_{l}  \right\|_{2}^{2}\\
	&=g(\mathcal{X}, \mathcal{V}^{(t+1)}).
	\end{aligned}
	\end{equation}
\end{proof} 
The first line of the proof says that, at iteration $t$ we have a set of cluster assignments for the unlabelled data $\Omega_{U}^{(t)}$ and for the labelled data $\Omega_{L}^{(t)}$ that satisfies all constraints imposed upon knowing the true classes of the points in $\mathcal{X}_{L}$. When we proceed into the next step of updating the set of bases $\mathcal{V}^{(t+1)}$ at iteration $t+1$, the new set of bases minimise the reconstruction error within each cluster of points given the assignment $\Omega^{(t)}$ (as stated in the second and third lines of the proof). Next in the assignment update stage for the unlabelled data, we obtain the fourth line of the proof. It says that the assignment $k_{u}^{(t+1)}$ in the $(t+1)$th iteration would only be different from $k_{u}^{(t)}$ if it gives a smaller reconstruction error for $\textbf{x}_{u}$. Finally in the last step of the KSCC algorithm, we update the matching between the cluster assignment and true classes, and it only gets updated if some other matching has a smaller overall reconstruction error for the labelled data $\mathcal{X}_{L}$. This is reflected in the last two lines of the proof. 

\begin{table*}[h!]
	\begin{center}
		\begin{tabular}{ |c|c|c|c|c|c|c|c| } 
			\hline
			Parameters & SCAL & SCAL-A & SCAL-D&MaxResid&MinMargin&Random\\
			\hline
			$\sigma=0.2$ &\textbf{0.30\%}&0.40\%&45.20\%&19.20\%&0.70\%&23.00\%\\
			$\sigma=0.4$  &\textbf{43.10\%}&46.10\%&99.00\%&98.00\%&83.10\%&99.50\%\\  
			$\sigma=0.6$  &85.60\% &\textbf{85.40\%}  &99.90\%&99.10\%&89.50\%&99.50\%\\
			\hline
			\hline
			$\theta=30$ &\textbf{41.67\%}&44.17\%&98.67\%&99.83\%&96.00\%&99.00\%\\
			$\theta=50$ &37.17\%&\textbf{36.83\%}&99.00\%&98.17\%&69.50\%&99.50\%\\  
			$\theta=70$ &32.17\%&\textbf{31.83\%}&98.83\%&98.50\%&77.67\%&99.83\%\\
			\hline
		\end{tabular}
	\end{center}
	\caption{The percentage of points queried before perfect cluster performance is reached on synthetic data sets.}
	\label{synthetic_data}
\end{table*}

\begin{algorithm}[h!]
	\caption{KSC with Constraints~(KSCC)}
	\DontPrintSemicolon
	\SetAlgoLined
	\SetKwInOut{Input}{Input}
	\SetKwInOut{Output}{Output}
	\Input{Labelled and unlabelled data $\mathcal{X}_{L}$, $\mathcal{X}_{U}$ \\
		Initial cluster assignment $\left\{k_{1},\ldots,k_{N} \right\}$\\
		subspace dimension $q$}
	\BlankLine
	\Repeat{\text Iteration number $T$ is reached or the total reconstruction error stops decreasing}{
		\% \emph{Stage 1: fitting subspaces}\;
		\For{$X_{k}\in\mathcal{X}$ ($k=1,\ldots,K$)}{
			Calculate the eigen-decomposition on the covariance matrix of $X_{k}$:\\
			$\text{cov}(X_{k})=V_{k}\Lambda_{k} V_{k}^{T}$ \;
		}
		\% \emph{Stage 2: updating assignment}\;
		\For{$\textbf{x}_{i}\in\mathcal{X}$ ($i=1,\ldots, N$)}{
			Determine the cluster assignment for $\textbf{x}_{i}$:\\
			$k_{i}=\arg\min_{k_{i}\in\left\{1,\ldots,K \right\}}\left\|\textbf{x}_{i}-V_{k}V_{k}^{T}\textbf{x}_{i} \right\|_{2}^{2}$
		}
		\% \emph{Stage 3: satisfying constraints}\;
		Find the best vector $P_{n}^{*}\in\mathcal{P}(K)$ to match with the true classes by solving Eq. \eqref{eq_updatel}:
		$$P_{n}^{*} =\arg \min_{\substack{P_{n}\in\mathcal{P}(K),\\n\in\left\{1,\ldots,K! \right\}}}\left\{\sum_{l=1}^{K} \left\|X_{l}-V_{P_{nl}}V_{P_{nl}}^{T}X_{l}  \right\|_{2}^{2} \right\}$$
		using the Hungarian algorithm\;
		\BlankLine
	}
	
	\label{algo_update}
\end{algorithm}

\section{Experimental Results} \label{al_er}
In this section, we conduct a series of experiments with both synthetic and real data to evaluate the performance of our proposed active learning strategies against three other competing strategies. The cluster performance is measured by the well-known normalised mutual information (NMI) \cite{cover2012elements}.



In order to inspect the influence of data addition and deletion separately, we use three versions of our proposed active learning strategy: \emph{SCAL-A} and \emph{SCAL-D} only take into account the influence of data addition or deletion respectively, and \emph{SCAL} is the combined strategy. We compare the performance of our proposed active learning strategies with three alternative schemes: \emph{MaxResid}, \emph{MinMargin}  \cite{lipor2015margin}, and \emph{Random} strategy. \emph{MaxResid} selects data that have the largest reconstruction error to their corresponding subspaces. \emph{MinMargin} selects the data that are most equidistant to their two closest subspaces. Lastly as a benchmark, we compare to random sampling and satisfy the constraints using the proposed KSCC algorithm.


\subsection{Synthetic data}
For all synthetic experiments, we initialise the cluster assignments with the best initialisation (the one with the lowest reconstruction error) out of 50 runs of the KSC algorithm. We set the total number of iterations $T$ of the active learning procedure to be $N$, and one point is queried at a time. 

%


%

\begin{figure*}[h!]
	\centering	
	\includegraphics[height=.25\textwidth, width=.32\textwidth]{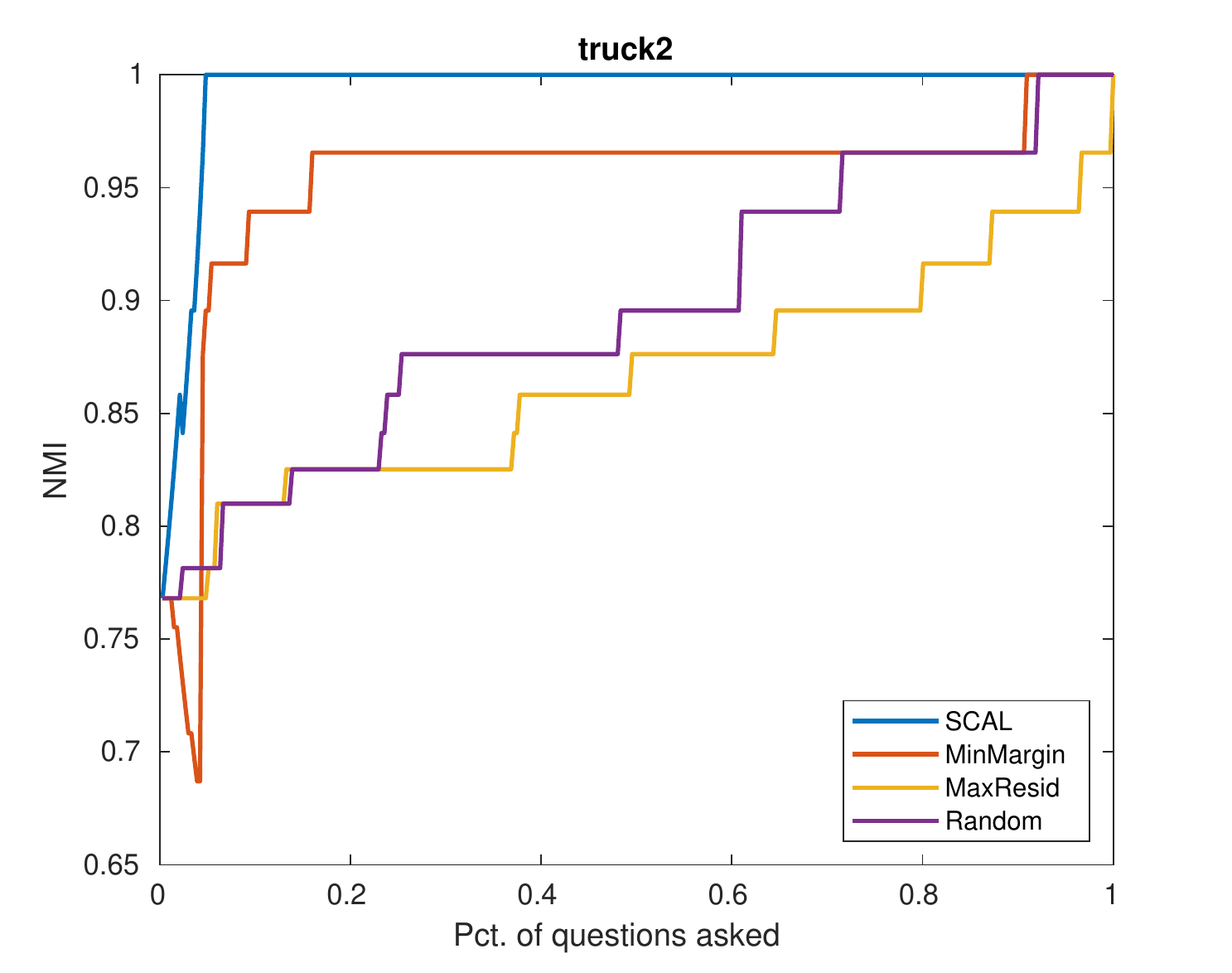}
	\includegraphics[height=.25\textwidth, width=.32\textwidth]{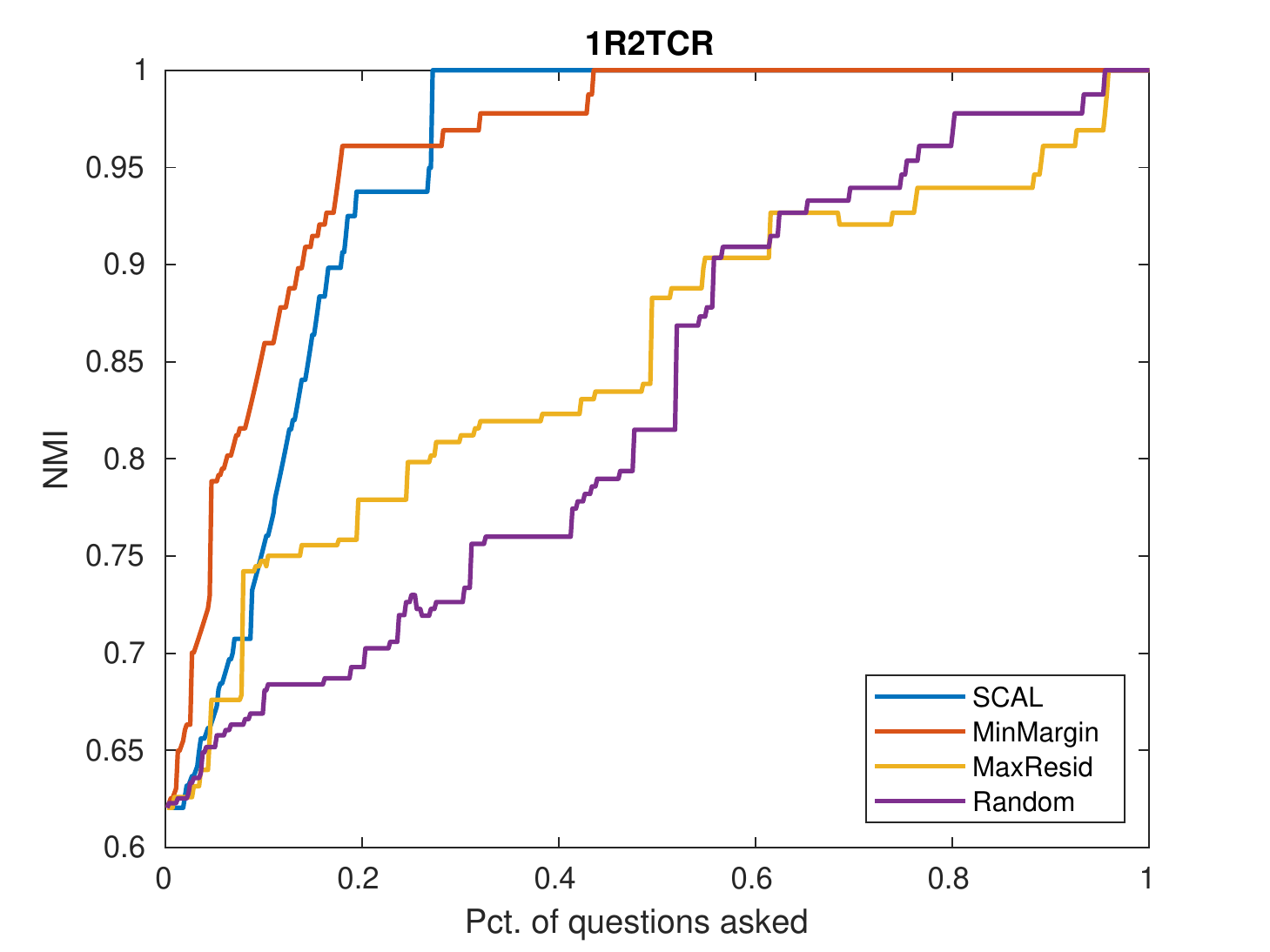}
	\includegraphics[height=.25\textwidth, width=.32\textwidth]{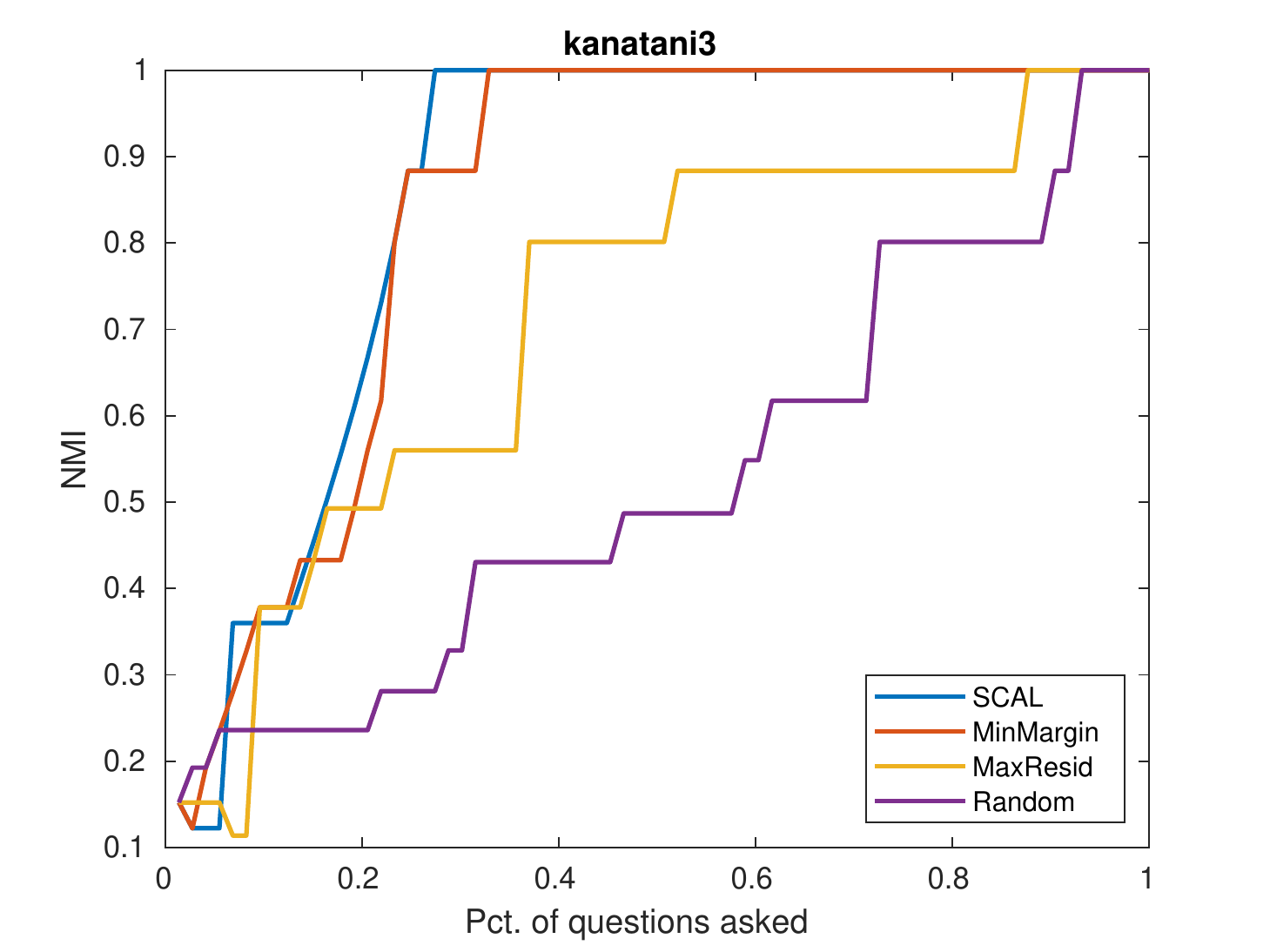}
	
	\includegraphics[height=.25\textwidth, width=.32\textwidth]{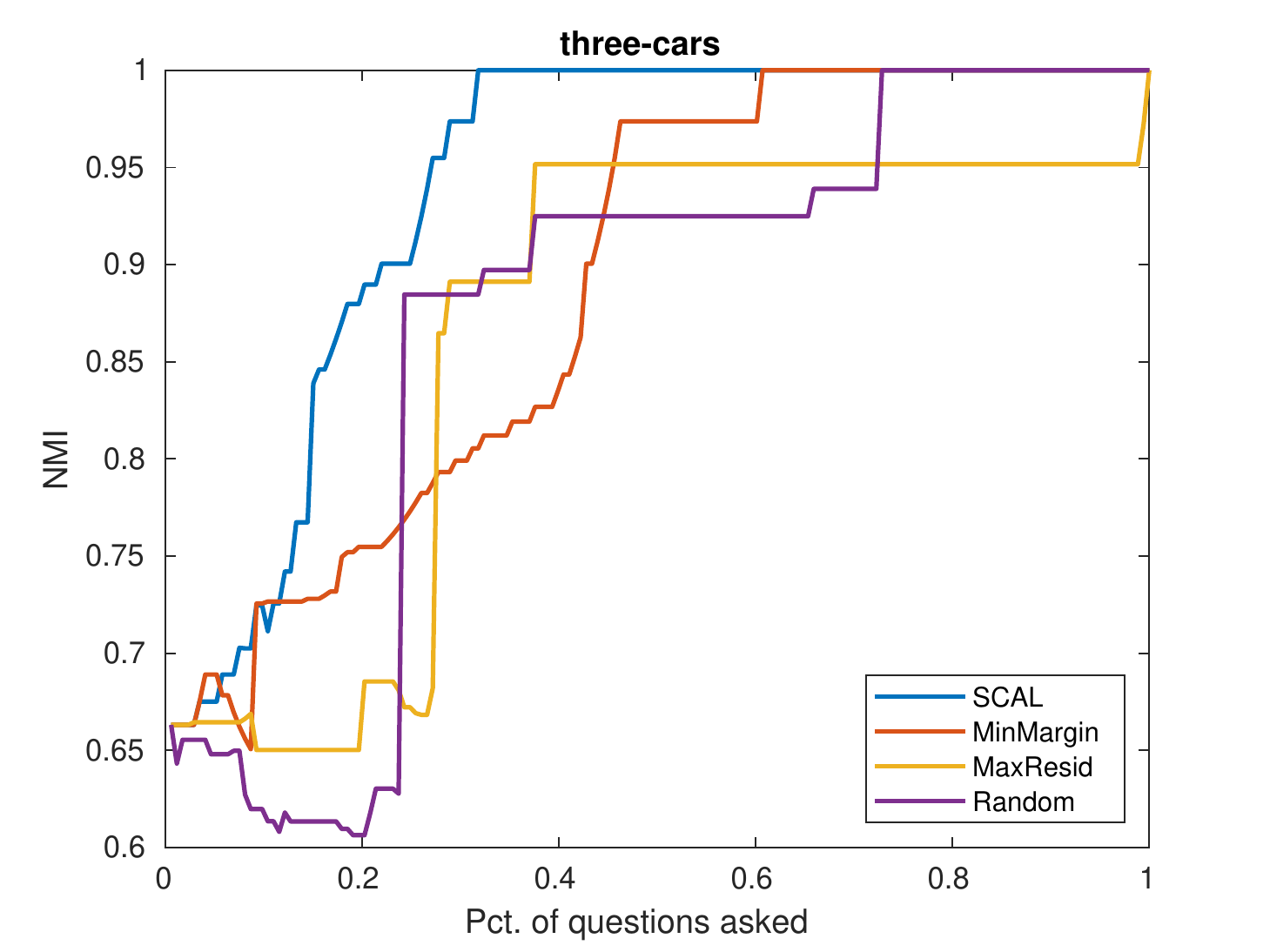}
	\includegraphics[height=.25\textwidth, width=.32\textwidth]{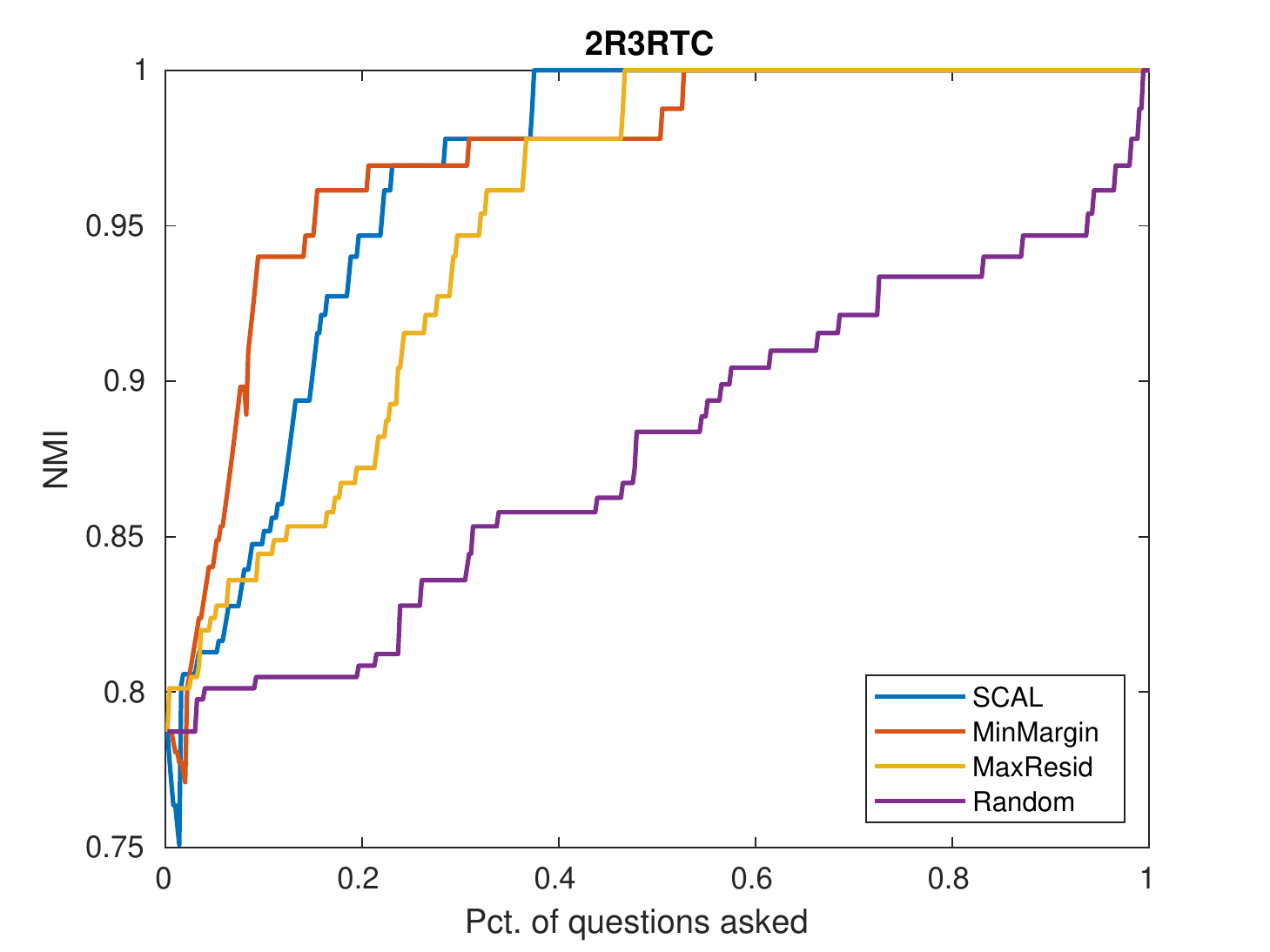}
	\includegraphics[height=.25\textwidth, width=.32\textwidth]{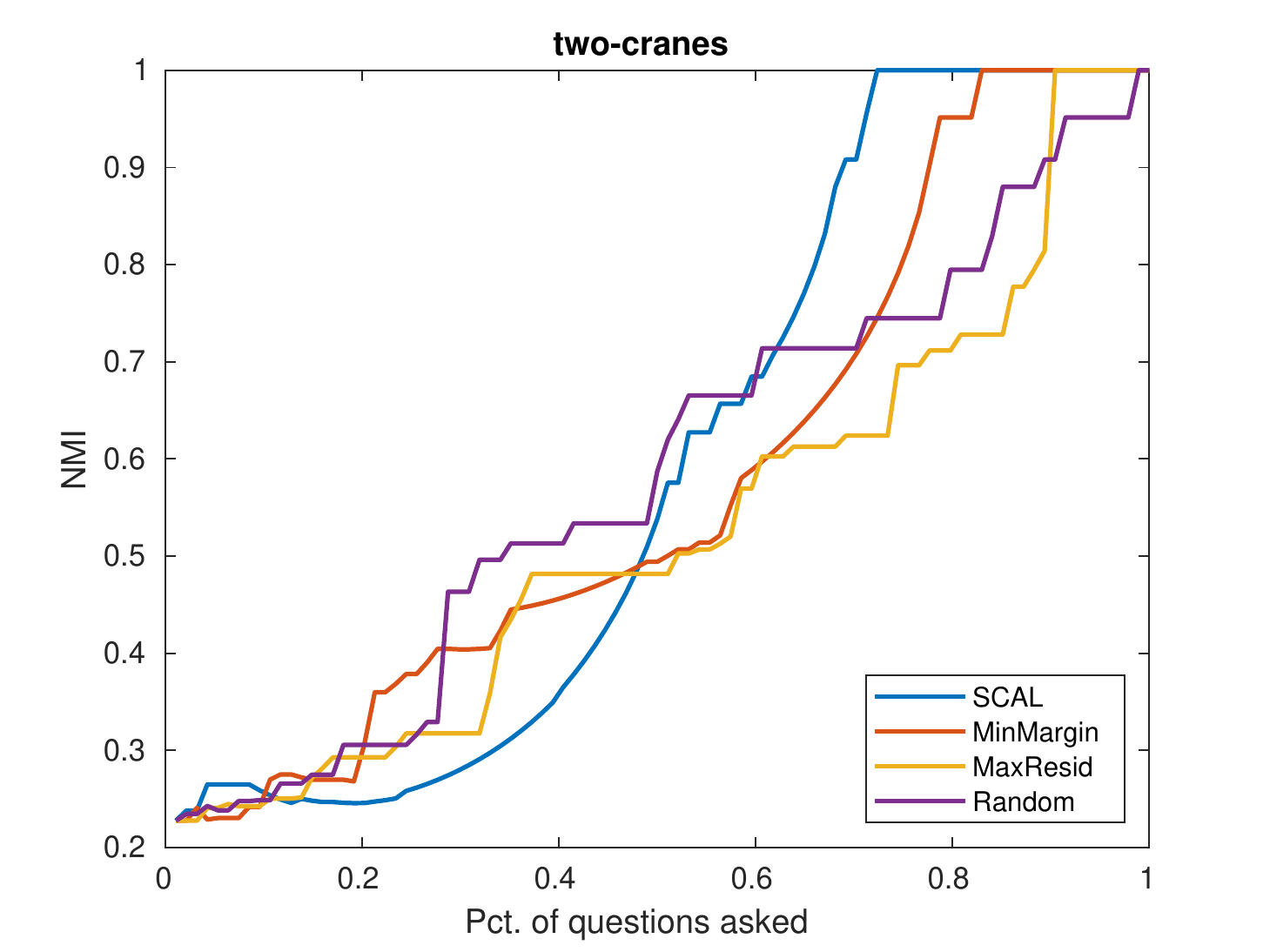}
	\caption{Performance results measured by NMI on six motion segmentation data sets with KSC initialisation.}
	\label{motion_figs}
\end{figure*}
\textbf{Experiment 1: Varying noise level ($\sigma$).} 
We investigate the effectiveness of our proposed active learning strategy under varying levels of additive noise. The effectiveness of various strategies are compared under various levels of additive noise. The data corrupted by noise can be expressed as $Y = X+E$, where $X$ is the noise-free data and $E$ the noise component. The noise-free data matrix $X$ are generated from the standard Normal distribution with $N(0,1)$ along their bases directions. Each entry in the noise data matrix $E$ is generated from standard Normal distribution $N(0,\sigma^{2})$, with zero mean and variance $\sigma^{2}$. We add additive noise levels of $\sigma=0.2$, 0.4, and 0.6 respectively. Across all noise levels, there are 5 clusters in each data set and each cluster contains 200 points from the same subspace of dimension 10 out of the full dimension 20.

The performance of various strategies under all settings are shown in Table \ref{synthetic_data}. As a general trend, all strategies require high proportion of points to be queried to reach perfection as the noise level goes up. It seems most of the advantage of \emph{SCAL} comes from the influence of data addition \emph{SCAL-A}, and it is difficult to say there is a difference between \emph{SCAL} and \emph{SCAL-A}. It is worth noting that \emph{MinMargin} has similar performance to \emph{SCAL} and \emph{SCAL-A} when $\sigma=0.2$.



\textbf{Experiment 2: Varying angles between subspaces ($\theta$).} In order to fix a vector to rotate the subspaces, we apply various active learning strategies on three 3-$D$ examples with subspace dimension $q=2$. 600 points are generated in total from 3 classes, and every cluster contains 200 points each. Noise with $\sigma=0.1$ is added to the data, and the between-subspace angle is specified to be 30, 50, and 70 degrees respectively. 

The performance results under all scenarios are shown in Table \ref{synthetic_data}. Our proposed active strategy \emph{SCAL} and \emph{SCAL-A} outperform all other strategies significantly. For these two strategies, the proportion of data needed in order to achieve perfection decreases as the between-subspace angle increases. Other strategies have to query almost all points in order to achieve perfect performance apart from \emph{MinMargin}, which is our close competitor in the varying noise setting.

Again \emph{SCAL-D} strategy as part of our proposed \emph{SCAL} strategy barely distinguishes itself from \emph{Random} strategy. This is within our expectations for two reasons. First, misclassified points are most likely to belong to the nearest cluster that they have the second least reconstruction error to. Secondly, those points whose deletion has a large influence on their allocated subspaces are likely to be correctly classified in the first place due to the noise in the data. 





\subsection{Real Data}
In this section, we conduct experiments on real-world data comparing \emph{SCAL} to various competing strategies. We show the advantage of our proposed active learning strategy when the data exhibit subspace structure. Specifically, we experiment with data sets in motion segmentation and  face clustering that have been used previously to demonstrate the effectiveness of subspace clustering \cite{elhamifar2013sparse}. 

For KSC-based experiments, we experiment with two initialisation schemes. First we initialise with the output given by KSC, which is the set of labels that gives the smallest reconstruction error out of 50 runs. Secondly, we initialise with the output from Sparse Subspace Clustering (SSC) \cite{elhamifar2013sparse} under the default model parameters. Due to the excellent performance of SSC, the aim is to investigate whether the correct initialisation of subspace bases would help accelerate the performance improvement.

All performance results are presented in two measures: the first row of each data set presents the percentage of data that need to be queried before perfect clustering is reached; the second row presents the area under the curve as a percentage of the total area. 

\begin{figure*}[h!]
	\centering
	\begin{subfigure}[t]{\textwidth}
		\centering
		\includegraphics[height=.26\textwidth, width=.32\textwidth]{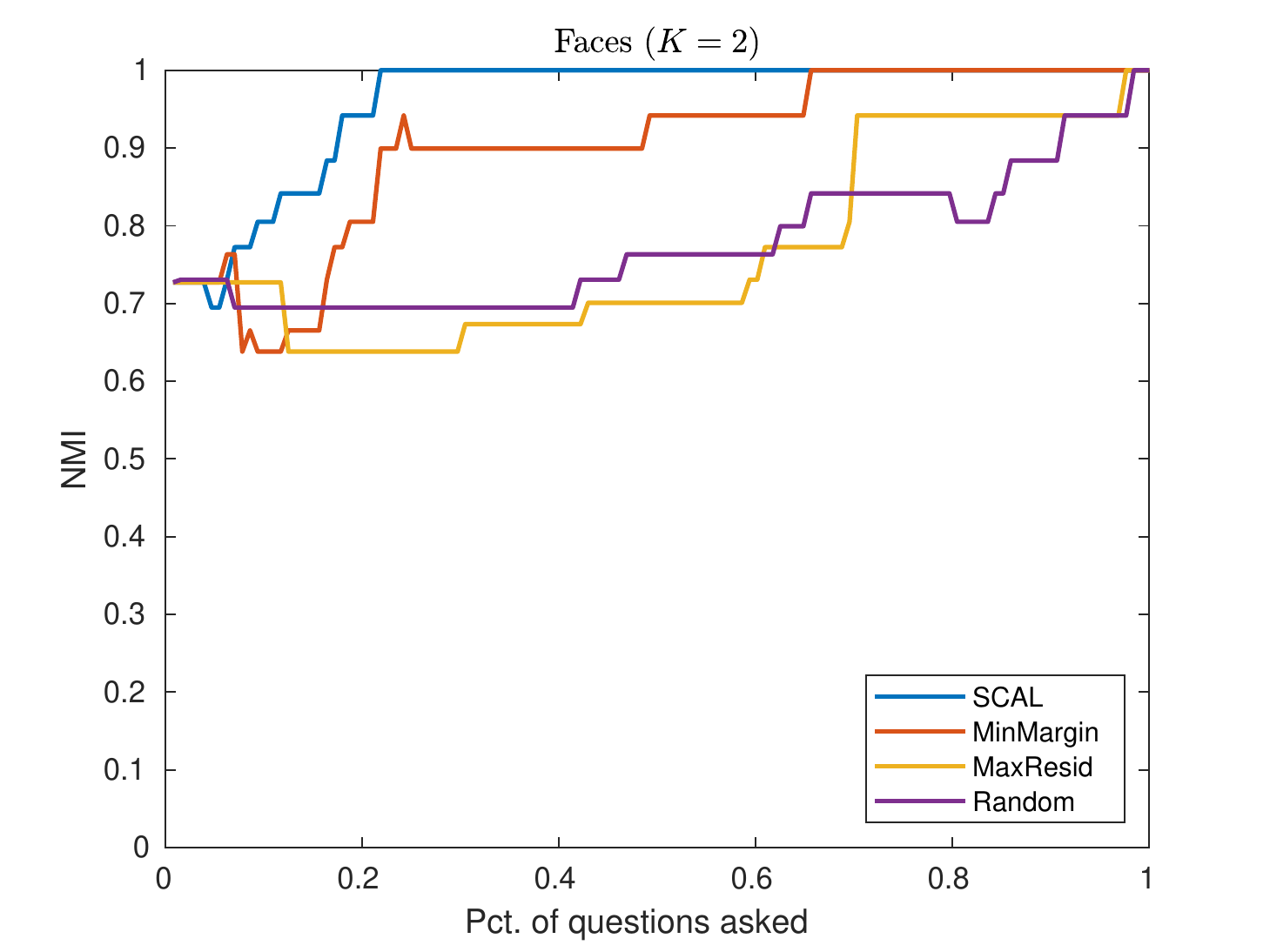}
		\includegraphics[height=.26\textwidth, width=.32\textwidth]{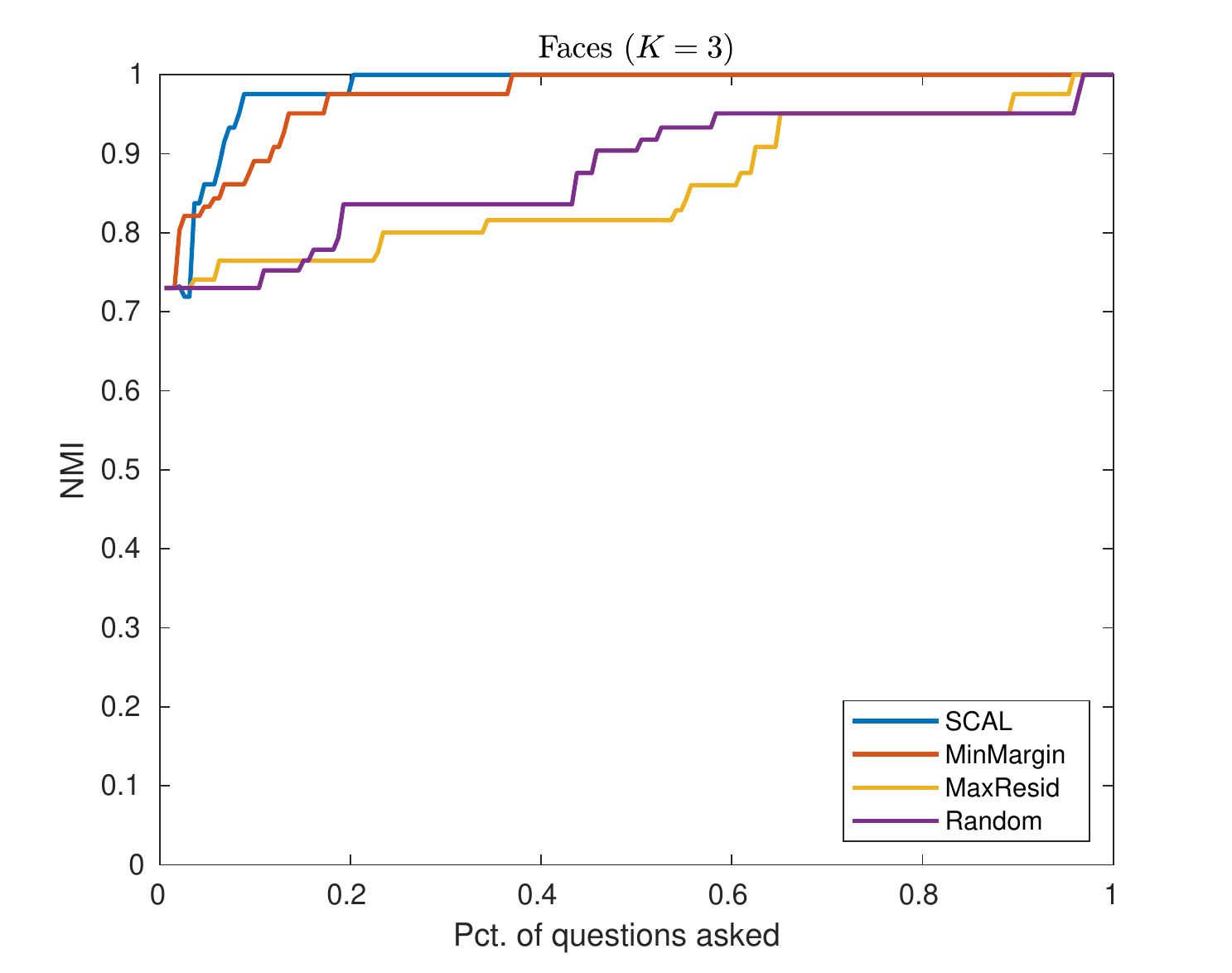}
		\includegraphics[height=.26\textwidth, width=.32\textwidth]{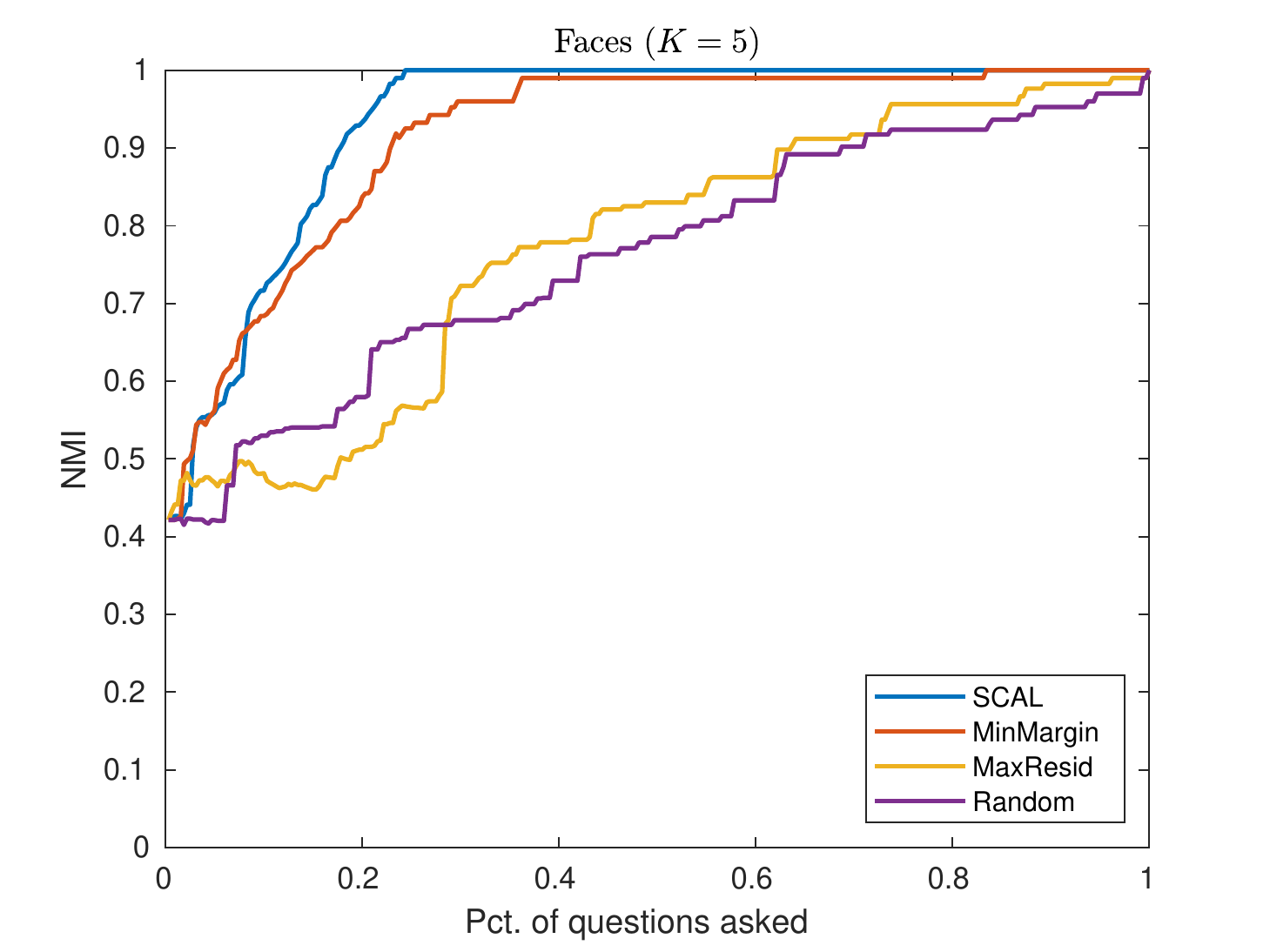}
	\end{subfigure}

	\begin{subfigure}[t]{\textwidth}
		\centering
		\includegraphics[height=.26\textwidth, width=.32\textwidth]{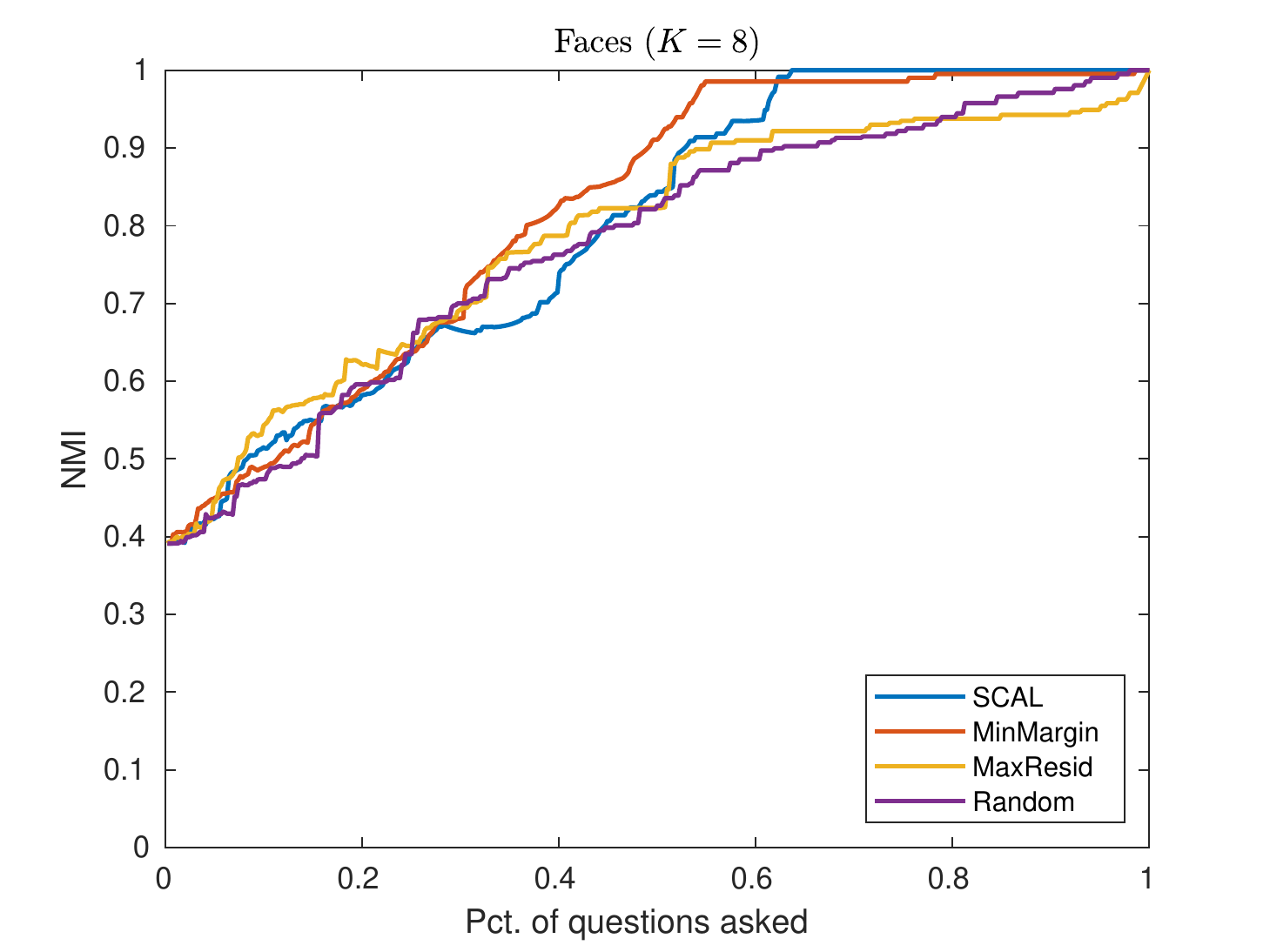}
		\includegraphics[height=.26\textwidth, width=.32\textwidth]{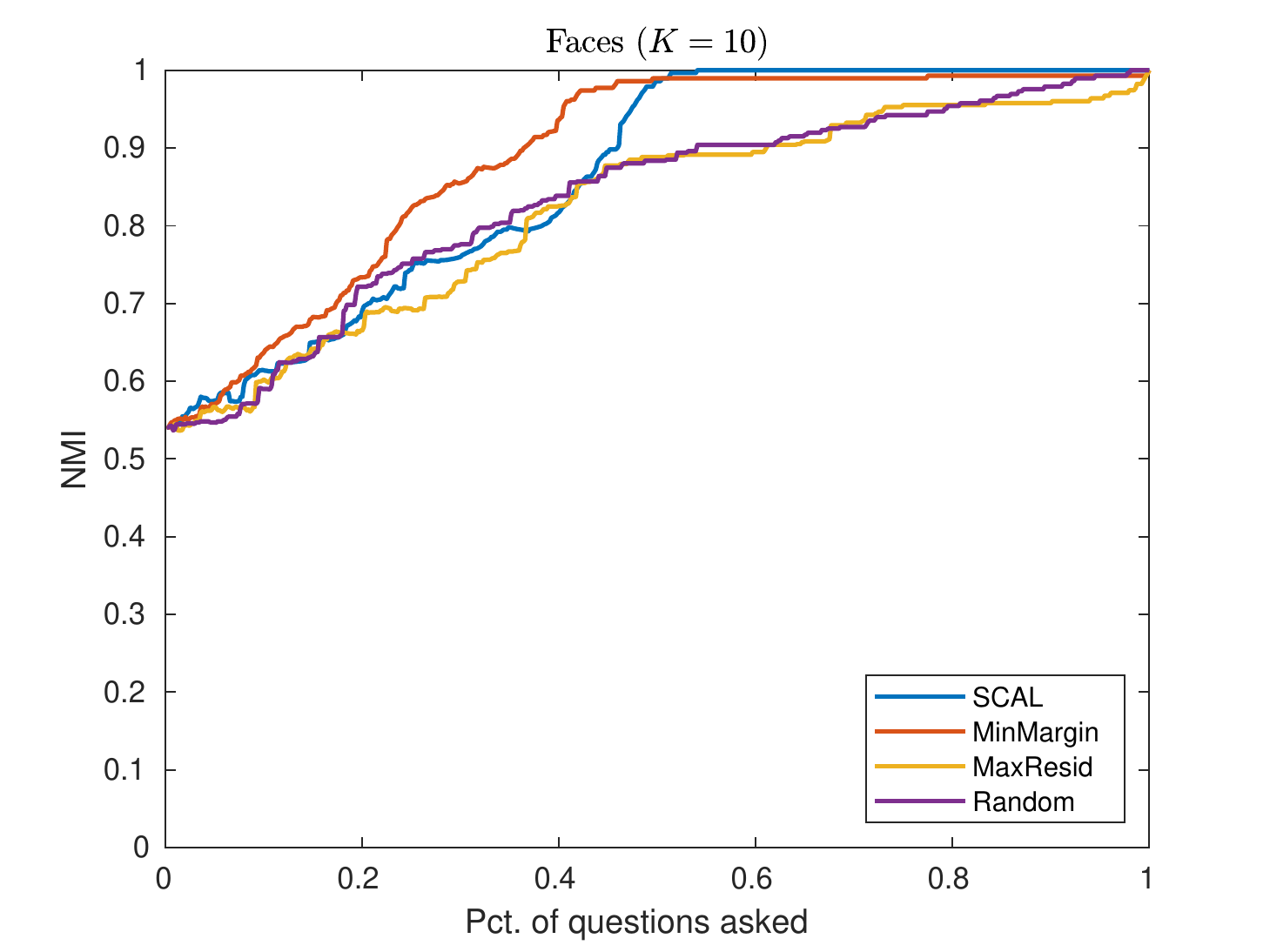}
	\end{subfigure}
	
	\caption{Performance results measured by NMI on Yale Faces datasets with KSC initialisation.} 
	\label{faces_vary}
\end{figure*}

\textbf{Motion segmentation.} In this set of experiments, we evaluate the performance of all strategies on six motion segmentation data sets \cite{tron2007benchmark}. Motion segmentation refers to the problem of separating the points in a sequence of frames that compose one video after being combined consecutively. Each point can be represented by a $2F$-dimensional vector, in which $F$ is the number of frames in the video \cite{elhamifar2013sparse}. 
Following the parameter setting from \cite{elhamifar2013sparse}, we set the subspace dimensionality $q=3$ and one point is queried at each iteration. 







The performance results are summarised in Table \ref{real_data2}. It is worth noting that SSC achieves perfect performance on `truck2' and `kanatani3', thus there is no need for active learning. The performance improvement over iterations is shown in Fig. \ref{motion_figs}. We see that the performance of \emph{MinMargin} is very similar to that of \emph{SCAL} most of the time. This is also reflected in the second row of the performance of each data set in Table \ref{real_data2}. However \emph{SCAL} always achieves perfect cluster performance first, which is what we expect to see due to its ability to query potentially misclassified points that are also informative. The performance of \emph{MaxResid} also improves rapidly in most scenarios, but it struggles to find the points that lead to perfect performance.


\begin{table}[h!]
	\begin{center}
		\begin{tabular}{ |c|c|c|c|c|c|} 
			\hline
			\multicolumn{5}{|c|}{\textbf{KSC update (KSC initialisation)}}\\
			\hline
			& \emph{SCAL} & \emph{MinMargin}  & \emph{MaxResid} &\emph{Random} \\
			\hline
			\multirow{2}{*}{truck2} &\textbf{4.53\%}&90.63\%&99.70\%&91.84\%\\
			&\textbf{99.36\%}&95.46\%&86.60\%&89.75\%\\
			\hline
			\multirow{2}{*}{1R2TCR}&\textbf{26.98\%}&43.35\%&95.68\%&88.13\%\\
			&94.89\%&\textbf{95.96\%}&85.44\%&83.12\%\\
			\hline
			\multirow{2}{*}{kanatani3} &\textbf{26.03\%}&31.51\%&86.30\%&91.78\%\\
			&\textbf{86.05\%}&84.86\%&72.10\%&53.10\%\\
			\hline
			\multirow{2}{*}{three-cars} &\textbf{31.21\%}&60.12\%&99.42\%&72.25\%\\
			&\textbf{94.18\%}&89.07\%&86.73\%&87.16\%\\
			\hline
			\multirow{2}{*}{2R3RTC}  &\textbf{37.28\%}&52.51\%&46.49\%&99.20\%\\
			&96.68\%&\textbf{97.23\%}&95.37\%&87.86\%\\
			\hline
			\multirow{2}{*}{two-cranes} &\textbf{71.28\%}&81.92\%&89.36\%&97.87\%\\
			&\textbf{59.52\%}&58.09\%&53.09\%&58.49\%\\
			\hline	
			\hline
			\multicolumn{5}{|c|}{\textbf{KSC update (SSC initialisation)}}\\
			\hline	
			\multirow{2}{*}{1R2TCR}&\textbf{26.98\%}&43.35\%&95.68\%&98.02\%\\
			&94.89\%&\textbf{95.96\%}&85.44\%&83.09\%\\
			\hline
			\multirow{2}{*}{three-cars} &\textbf{1.73\%}&83.82\%&99.42\%&84.39\%\\
			&\textbf{99.94\%}&97.61\%&95.52\%&97.49\%\\
			\hline
			\multirow{2}{*}{2R3RTC} &\textbf{9.82\%}&40.88\%&59.92\%&78.96\%\\
			&\textbf{99.76\%}&99.27\%&98.21\%&97.73\%\\
			\hline
			\multirow{2}{*}{two-cranes} &\textbf{73.40\%}&75.53\%&98.94\%&98.94\%\\
			&64.38\%&\textbf{71.23\%}&45.72\%&65.80\%\\
			
			\hline	
			
		\end{tabular}
	\end{center}
	\caption{Performance on motion segmentation data sets.}
	\label{real_data2}
\end{table}

\textbf{Face clustering.} The original Extended Yale Face Database~B~\cite{lee2005acquiring} consists of 64 images of 38 distinct faces under various lighting conditions. Each original image is of size 192 $\times$ 168, and have been downsampled to have size 48 $\times$ 42 \cite{elhamifar2013sparse}. It has previously been shown under the Lambertian assumption that images of a subject lie close to a linear subspace of dimension 9 \cite{basri2003lambertian}. Since the data have intrinsically low dimension, we preprocess the data by projecting onto its first 5$K$ principal components as has been done in \cite{balcan2007margin}. Following the experimental settings in \cite{elhamifar2013sparse}, we experiment with $K=2$, 3, 5, 8, and 10. The corresponding data sets are obtained from the SSC package in MATLAB \cite{elhamifar2013sparse}.



As before, we apply all active learning strategies with both KSC and SSC initialisations. It is worth noting that SSC achieves perfect performance on the preprocessed data when $K=2$, thus there is no need for active learning. From the results shown in Table \ref{faces_data}, we see that the percentage of data needed goes up as $K$ increases. Although the proportion of area under the curve is very similar between \emph{MinMargin} and \emph{SCAL}, \emph{MinMargin} requires a much higher percentage of queries than \emph{SCAL} before perfect clustering is reached.

The performance improvement over time with KSC initialisation is shown in Fig. \ref{faces_vary}. The initial performance decreases slowly as $K$ increases, and the performance of various active strategies gets closer. 
With that said, the performance results of \emph{SCAL} and \emph{MinMargin} still stand out from the rest. 


\section{Extension to Spectral Clustering}
Finally, we make an initial attempt to extend our active learning framework to the spectral clustering setting. A large number of subspace clustering algorithms are spectral-based methods. These methods construct a pairwise affinity matrix through various optimisation schemes and solve the cluster assignment problem through spectral clustering \cite{elhamifar2013sparse, liu2010robust}. 
We incorporate the queried information in the similarity matrix and compare with other strategies in the spectral setting.


Using our proposed strategy, the points are queried in the same manner as before.
Upon receiving the class information of some points, the constraints are satisfied by editing the affinity matrix. 
Following the update procedure in \cite{lipor2017leveraging}, we set to ones for those labelled data that belong to the same class and zeros for those that lie in different classes. 
Spectral clustering is then applied to the edited affinity matrix to obtain labels for all points. As a final step, we apply KSCC to ensure that all grouping constraints are satisfied for the labelled data.

The performance results are shown in the last section of Table \ref{faces_data}. Note that the authors that propose \emph{MinMargin} have renamed it to \emph{SUPERPAC} for the spectral setting. \emph{SCAL} outperforms other competing strategies in all scenarios apart from when $K=10$. With that said, \emph{SCAL} still enjoys the same level of rate of improvement as \emph{SUPERPAC} and \emph{MaxResid} when $K=10$.

\begin{table}
	\begin{center}
		\begin{tabular}{ |c|c|c|c|c|c|c| } 
			\hline

			\multicolumn{6}{|c|}{\textbf{KSC update (KSC initialisation)}}\\
			\hline
			Strategies & $K=2$ & $K=3$ & $K=5$&$K=8$&$K=10$\\
			\hline
			\multirow{2}{*}{\emph{SCAL}} &\textbf{21.09\%}&\textbf{19.79\%}&\textbf{24.06\%}&\textbf{63.48\%}&\textbf{53.91\%}\\
			&\textbf{96.10\%}&\textbf{98.28\%}&\textbf{93.80\%}&80.02\%&86.45\%\\
			\hline
			\multirow{2}{*}{\emph{MinMargin}}  &64.84\%&36.46\%&83.13\%&98.44\%&99.84\%\\
			&90.30\%&97.31\%&91.83\%&\textbf{81.92\%}&\textbf{88.95\%}\\
			\hline
			\multirow{2}{*}{\emph{MaxResid}}  &96.88\%&95.31\%&99.69\%&99.81\%&99.84\%\\
			&77.19\%&85.85\%&77.43\%&79.01\%&82.22\%\\
			\hline
			\multirow{2}{*}{\emph{Random}}  &97.66\%&96.35\%&99.69\%&97.85\%&97.97\%\\
			&77.59\%&88.08\%&76.27\%&77.77\%&83.46\%\\
			\hline

		\end{tabular}
	\end{center}
	
	\begin{center}
		\begin{tabular}{ |c|c|c|c|c|c| } 
			\hline
			\multicolumn{5}{|c|}{\textbf{KSC update (SSC initialisation)}}\\
			\hline
			Strategies  & $K=3$ & $K=5$&$K=8$&$K=10$\\
			\hline
			\multirow{2}{*}{\emph{SCAL}}  &\textbf{10.94\%}&\textbf{25.31\%}&\textbf{28.52\%}&\textbf{58.91\%}\\
			&\textbf{99.06\%}&\textbf{99.09\%}&\textbf{98.24\%}&95.15\%\\
			\hline
			\multirow{2}{*}{\emph{MinMargin}}   &16.15\%&38.13\%&92.58\%&62.18\%\\
			&98.59\%&98.14\%&98.00\%&\textbf{97.15\%}\\
			\hline
			\multirow{2}{*}{\emph{MaxResid}}  &93.23\%&99.69\%&99.81\%&99.38\%\\
			&91.50\%&82.66\%&89.75\%&93.34\%\\
			\hline
			\multirow{2}{*}{\emph{Random}}   &91.15\%&76.25\%&99.02\%&91.88\%\\
			&92.19\%&94.40\%&91.59\%&94.75\%\\
			\hline

		\end{tabular}
	\end{center}
	\begin{center}
		\begin{tabular}{ |c|c|c|c|c|c| } 
			\hline
			\multicolumn{5}{|c|}{\textbf{Spectral update (SSC initialisation)}}\\
			\hline
			Strategies  & $K=3$ & $K=5$&$K=8$&$K=10$\\
			\hline
			\multirow{2}{*}{\emph{SCAL}} &\textbf{10.94\%}&\textbf{26.56\%}&\textbf{40.62\%}&26.56\%\\
			&\textbf{98.91\%}&\textbf{96.03\%}&\textbf{91.08\%}&98.24\%\\
			\hline
			\multirow{2}{*}{\emph{SUPERPAC}} &14.06\%&31.25\%&\textbf{40.62\%}&\textbf{20.31\%}\\
			&98.39\%&95.96\%&90.03\%&98.67\%\\
			\hline
			\multirow{2}{*}{\emph{MaxResid}} &90.62\%&98.44\%&98.44\%&57.81\%\\
			&93.23\%&92.84\%&86.81\%&\textbf{98.76\%}\\
			\hline
			\multirow{2}{*}{\emph{Random}}  &95.31\%&98.44\%&92.19\%&96.88\%\\
			&90.90\%&92.65\%&87.50\%&96.06\%\\
			\hline	
		\end{tabular}
	\end{center}
	\caption{Performance on Yale Faces data sets.}
	\label{faces_data}
\end{table}

\section{Conclusions \& Future Work} \label{al_conclusion}
We proposed a novel active learning framework for subspace clustering. Ideas from matrix perturbation theory are borrowed to enable efficient estimation of the influence of data deletion and addition as measured by the change in the reconstruction error. New results on the perturbation analysis of data addition are provided as a by-product of our proposed active learning framework. In addition, we propose a constrained subspace clustering algorithm KSCC that monotonically decreases the constrained objective over iterations. 

For future research, we plan to consider extensions of our proposed framework in the spectral setting. Our initial experimental results seem promising on the Faces data using the straightforward spectral update on the affinity matrix. 


%
\section*{Acknowledgment}
Hankui Peng acknowledges the financial support of Lancaster
University and the Office for National Statistics Data Science
Campus as part of the EPSRC-funded STOR-i Centre for
Doctoral Training.

\bibliographystyle{IEEEtran}

\begin{thebibliography}{10}
	\providecommand{\url}[1]{#1}
	\csname url@samestyle\endcsname
	\providecommand{\newblock}{\relax}
	\providecommand{\bibinfo}[2]{#2}
	\providecommand{\BIBentrySTDinterwordspacing}{\spaceskip=0pt\relax}
	\providecommand{\BIBentryALTinterwordstretchfactor}{4}
	\providecommand{\BIBentryALTinterwordspacing}{\spaceskip=\fontdimen2\font plus
		\BIBentryALTinterwordstretchfactor\fontdimen3\font minus
		\fontdimen4\font\relax}
	\providecommand{\BIBforeignlanguage}[2]{{%
			\expandafter\ifx\csname l@#1\endcsname\relax
			\typeout{** WARNING: IEEEtran.bst: No hyphenation pattern has been}%
			\typeout{** loaded for the language `#1'. Using the pattern for}%
			\typeout{** the default language instead.}%
			\else
			\language=\csname l@#1\endcsname
			\fi
			#2}}
	\providecommand{\BIBdecl}{\relax}
	\BIBdecl
	
	\bibitem{su2012crowdsourcing}
	H.~Su, J.~Deng, and L.~Fei-Fei, ``Crowdsourcing annotations for visual object
	detection,'' in \emph{Workshops at the Twenty-Sixth AAAI Conference on
		Artificial Intelligence}, 2012.
	
	\bibitem{settles2008curious}
	B.~Settles, ``Curious machines: Active learning with structured instances,''
	Ph.D. dissertation, University of Wisconsin--Madison, 2008.
	
	\bibitem{lipor2015margin}
	J.~Lipor and L.~Balzano, ``Margin-based active subspace clustering,'' in
	\emph{2015 IEEE 6th International Workshop on Computational Advances in
		Multi-Sensor Adaptive Processing (CAMSAP)}, 2015.
	
	\bibitem{lipor2017leveraging}
	J.~Lipor and L.~Balzano, ``Leveraging union of subspace structure to improve constrained
	clustering,'' in \emph{Proceedings of the 34th International Conference on
		Machine Learning-Volume 70}.\hskip 1em plus 0.5em minus 0.4em\relax JMLR.
	org, 2017.
	
	\bibitem{bradley2000k}
	P.~S. Bradley and O.~L. Mangasarian, ``K-plane clustering,'' \emph{Journal of
		Global Optimization}, vol.~16, no.~1, pp. 23--32, 2000.
	
	\bibitem{jolliffe2011principal}
	I.~Jolliffe, ``Principal component analysis,'' in \emph{International
		encyclopedia of statistical science}.\hskip 1em plus 0.5em minus 0.4em\relax
	Springer, 2011, pp. 1094--1096.
	
	\bibitem{critchley1985influence}
	F.~Critchley, ``Influence in principal components analysis,''
	\emph{Biometrika}, vol.~72, no.~3, pp. 627--636, 1985.
	
	\bibitem{balcan2007margin}
	M.-F. Balcan, A.~Broder, and T.~Zhang, ``Margin based active learning,'' in
	\emph{International Conference on Computational Learning Theory}, 2007.
	
	\bibitem{seung1992query}
	H.~S. Seung, M.~Opper, and H.~Sompolinsky, ``Query by committee,'' in
	\emph{Proceedings of the fifth annual workshop on Computational learning
		theory}, 1992.
	
	\bibitem{settles2008multiple}
	B.~Settles, M.~Craven, and S.~Ray, ``Multiple-instance active learning,'' in
	\emph{Advances in neural information processing systems}, 2008.
	
	\bibitem{roy2001toward}
	N.~Roy, and A.~McCallum, ``Toward optimal active learning through monte carlo estimation of error reduction,'' in
	\emph{ICML}, 2001.
	
	\bibitem{donmez2007dual}
	P.~Donmez, J.~G. Carbonell, and P.~N. Bennett, ``Dual strategy active
	learning,'' in \emph{European Conference on Machine Learning}, 2007.
	
	\bibitem{melville2004diverse}
	P.~Melville and R.~J. Mooney, ``Diverse ensembles for active learning,'' in
	\emph{Proceedings of the twenty-first international conference on Machine
		learning}, 2004.
	
	\bibitem{nguyen2004active}
	H.~T. Nguyen and A.~Smeulders, ``Active learning using pre-clustering,'' in
	\emph{Proceedings of the twenty-first international conference on Machine
		learning}, 2004.
	
	\bibitem{freund1997selective}
	Y.~Freund, H.~S. Seung, E.~Shamir, and N.~Tishby, ``Selective sampling using
	the query by committee algorithm,'' \emph{Machine learning}, vol.~28, no.
	2-3, pp. 133--168, 1997.
	
	\bibitem{shi1997local}
	L.~Shi, ``Local influence in principal components analysis,''
	\emph{Biometrika}, vol.~84, no.~1, pp. 175--186, 1997.
	
	\bibitem{enguix2005influence}
	A.~Enguix-Gonz{\'a}lez, J.~Mu{\~n}oz-Pichardo, J.~Moreno-Rebollo, and
	R.~Pino-Mej{\'\i}as, ``Influence analysis in principal component analysis
	through power-series expansions,'' \emph{Communications in
		Statistics—Theory and Methods}, vol.~34, no. 9-10, pp. 2025--2046, 2005.
	
	\bibitem{wang1993effects}
	S.-G. Wang and E.~P. Liski, ``Effects of observations on the eigensystem of a
	sample covariance matrix,'' \emph{Journal of statistical planning and
		inference}, vol.~36, no. 2-3, pp. 215--226, 1993.
	
	\bibitem{benasseni2018correction}
	J.~B{\'e}nass{\'e}ni, ``A correction of approximations used in sensitivity
	study of principal component analysis,'' \emph{Computational Statistics},
	2018.
	
	\bibitem{golub2012matrix}
	G.~H. Golub and C.~F. Van~Loan, \emph{Matrix computations}. JHU Press, 2012.
	
	\bibitem{kuhn1955hungarian}
	H.~W. Kuhn, ``The hungarian method for the assignment problem,'' \emph{Naval
		research logistics quarterly}, vol.~2, no. 1-2, pp. 83--97, 1955.
	
	\bibitem{cover2012elements}
	T.~M. Cover and J.~A. Thomas, \emph{Elements of information theory}. John Wiley \& Sons, 2012.
	
	\bibitem{elhamifar2013sparse}
	E.~Elhamifar and R.~Vidal, ``Sparse subspace clustering: Algorithm, theory, and
	applications,'' \emph{IEEE transactions on pattern analysis and machine
		intelligence}, vol.~35, no.~11, pp. 2765--2781, 2013.
	
	\bibitem{tron2007benchmark}
	R.~Tron and R.~Vidal, ``A benchmark for the comparison of 3-d motion
	segmentation algorithms,'' in \emph{2007 IEEE conference on computer vision
		and pattern recognition}.\hskip 1em plus 0.5em minus 0.4em\relax IEEE, 2007,
	pp. 1--8.
	
	\bibitem{lee2005acquiring}
	K.-C. Lee, J.~Ho, and D.~J. Kriegman, ``Acquiring linear subspaces for face
	recognition under variable lighting,'' \emph{IEEE Transactions on Pattern
		Analysis \& Machine Intelligence}, no.~5, pp. 684--698, 2005.
	
	\bibitem{basri2003lambertian}
	R.~Basri and D.~W. Jacobs, ``Lambertian reflectance and linear subspaces,''
	\emph{IEEE Transactions on Pattern Analysis \& Machine Intelligence}, no.~2,
	pp. 218--233, 2003.
	
	\bibitem{liu2010robust}
	G.~Liu, Z.~Lin, and Y.~Yu, ``Robust subspace segmentation by low-rank
	representation,'' in \emph{Proceedings of the 27th international conference
		on machine learning (ICML-10)}, 2010, pp. 663--670.
	
	\bibitem{wang2010active}
	X.~Wang and I.~Davidson, ``Active spectral clustering,'' in \emph{2010 IEEE
		International Conference on Data Mining}.\hskip 1em plus 0.5em minus
	0.4em\relax IEEE, 2010, pp. 561--568.
	
	\bibitem{xiong2017active}
	C.~Xiong, D.~M. Johnson, and J.~J. Corso, ``Active clustering with model-based
	uncertainty reduction,'' \emph{IEEE transactions on pattern analysis and
		machine intelligence}, vol.~39, no.~1, pp. 5--17, 2017.
	
\end{thebibliography}

\appendix

\noindent\textbf{Proof of Proposition 1}

In this section, we provide the proof for Proposition 1 in Section \ref{query} of the paper regarding the perturbed form of the covariance matrix after data addition. 

\textbf{Proposition 1.} The form of $S_{(I)}^{+}$ can be expressed as,
\begin{equation}
\begin{split}
S_{(I)}^{+}=S+\frac{l}{n+l}\left[(S_{I}-S)-(\bar{\textbf{x}}_{I}+\bar{\textbf{x}})(\bar{\textbf{x}}_{I}+\bar{\textbf{x}})^{T} \right]\\
+\frac{l^{2}}{(n+l)^{2}}\left(\bar{\textbf{x}}_{I}+\bar{\textbf{x}} \right)\left(\bar{\textbf{x}}_{I}+\bar{\textbf{x}} \right)^{T}.
\end{split}
\end{equation}

\begin{proof}
	The sample covariance matrices $S$, $S_{I}$, and $S_{(I)}^{+}$ are given as
	\begin{equation}
	\begin{aligned}
	&S=\frac{1}{n}X^{T}\left(\textbf{I}_{n}-\frac{1}{n}\textbf{1}_{n}\textbf{1}_{n}^{T} \right)X,\\
	&S_{I}=\frac{1}{l}X_{I}^{T}\left(\textbf{I}_{l}-\frac{1}{l}\textbf{1}_{l}\textbf{1}_{l}^{T} \right)X_{I},\\
	&S_{(I)}^{+}=\frac{1}{n+l}X_{I_{+}}^{T}\left(\textbf{I}_{n+l}-\frac{1}{n+l}\textbf{1}_{n+l}\textbf{1}_{n+l}^{T} \right)X_{I_{+}},
	\end{aligned}
	\end{equation}
	in which $\textbf{I}_{n}$ is an identity matrix of size $n\times n$ and $\textbf{1}_{n}$ is a vector of all ones that has length $n$. Starting from the expression for $S_{(I)}^{+}$ above, we can write:
	\begin{equation*}
	\begin{aligned}
	(n+l)S_{(I)}^{+}
	&=X_{I_{+}}^{T}X_{I_{+}}-\frac{1}{n+l}X_{I_{+}}^{T}\textbf{1}_{n+l}\textbf{1}_{n+l}^{T}X_{I_{+}}\\
	&=X^{T}X+X_{I}^{T}X_{I}-\frac{1}{n+l} \bar{\textbf{x}}_{I_{+}}\bar{\textbf{x}}_{I_{+}}^{T}\\
	&=X^{T}X+X_{I}^{T}X_{I}-\frac{1}{n+l}(n\bar{\textbf{x}}+l\bar{\textbf{x}}_{I})(n\bar{\textbf{x}}+l\bar{\textbf{x}}_{I})^{T}\\
	&=nS+lS_{I}+\frac{1}{n}X^{T}\textbf{1}_{n}\textbf{1}_{n}^{T}X+\frac{1}{l}X_{I}^{T}\textbf{1}_{l}\textbf{1}_{l}^{T}X_{I}\\&\hspace{1.5em}-\frac{1}{n+l}(n\bar{\textbf{x}}+l\bar{\textbf{x}}_{I})(n\bar{\textbf{x}}+l\bar{\textbf{x}}_{I})^{T}\\
	&=nS+lS_{I}+n\bar{\textbf{x}}\bar{\textbf{x}}^{T}+l\bar{\textbf{x}}_{I}\bar{\textbf{x}}_{I}^{T}\\&\hspace{1.5em}-\frac{nl}{n+l}(\frac{n}{l}\bar{\textbf{x}}\bar{\textbf{x}}^{T}+\bar{\textbf{x}}\bar{\textbf{x}}_{I}^{T}+\bar{\textbf{x}}_{I}\bar{\textbf{x}}^{T}+\frac{l}{n}\bar{\textbf{x}}_{I}\bar{\textbf{x}}_{I}^{T})\\
	&=nS+lS_{I}-\frac{nl}{n+l}(\bar{\textbf{x}}\bar{\textbf{x}}^{T}+\bar{\textbf{x}}\bar{\textbf{x}}_{I}^{T}+\bar{\textbf{x}}_{I}\bar{\textbf{x}}^{T}+\bar{\textbf{x}}_{I}\bar{\textbf{x}}_{I}^{T})\\
	&=nS+lS_{I}-\frac{nl}{n+l}(\bar{\textbf{x}}_{I}+\bar{\textbf{x}})(\bar{\textbf{x}}_{I}+\bar{\textbf{x}})^{T}\\
	&=(n+l)S+l\left(S_{I}-S \right)-l(\bar{\textbf{x}}_{I}+\bar{\textbf{x}})(\bar{\textbf{x}}_{I}+\bar{\textbf{x}})^{T}\\&\hspace{1.5em}+\frac{l^{2}}{n+l}(\bar{\textbf{x}}_{I}+\bar{\textbf{x}})(\bar{\textbf{x}}_{I}+\bar{\textbf{x}})^{T},
	\end{aligned}
	\end{equation*}
	from which the result follows.
\end{proof}

\noindent\textbf{Proof of Proposition 2}

In this section, we provide the proof for Proposition 1 in Section \ref{query} of the paper regarding the perturbed form of the covariance matrix after the addition of one single point.

\textbf{Proposition 2.} Following the same line of analysis, we can show that the perturbed covariance matrix in the case when $l=1$ can be expressed as,
\begin{equation}
\begin{split}
S_{(i)}^{+}=S+\frac{1}{n+1}\left[\left(\bar{\textbf{x}}-\textbf{x}_{i} \right)\left(\bar{\textbf{x}}-\textbf{x}_{i} \right)^{T}-S \right]\\
-\frac{1}{(n+1)^2}\left(\bar{\textbf{x}}-\textbf{x}_{i} \right)\left(\bar{\textbf{x}}-\textbf{x}_{i} \right)^{T},
\end{split}
\end{equation}
in which we can think of $\frac{1}{n+1}$ and $\left[\left(\bar{\textbf{x}}-\textbf{x}_{i} \right)\left(\bar{\textbf{x}}-\textbf{x}_{i} \right)^{T}-S\right]$ as $\epsilon$ and $S^{(1)}$ respectively.

\begin{proof}
	\begin{equation*}
	\begin{aligned}
	(n+1)S_{(i)}^{+}
	&=X_{i_{+}}^{T}X_{i_{+}}-\frac{1}{n+1}X_{i_{+}}^{T}\textbf{1}_{n+1}\textbf{1}_{n+1}^{T}X_{i_{+}}\\
	&=X^{T}X+\textbf{x}_{i}\textbf{x}_{i}^{T}-\frac{1}{n+1}(n\bar{\textbf{x}}+\textbf{x}_{i})(n\bar{\textbf{x}}+\textbf{x}_{i})^{T}\\
	&=nS+\frac{1}{n}X^{T}\textbf{1}_{n}\textbf{1}_{n}^{T}X+\textbf{x}_{i}\textbf{x}_{i}^{T}\\&\hspace{1.5em}-\frac{1}{n+l}(n\bar{\textbf{x}}+l\bar{\textbf{x}}_{I})(n\bar{\textbf{x}}+l\bar{\textbf{x}}_{I})^{T}\\
	&=nS+\frac{1}{n}X^{T}\textbf{1}_{n}\textbf{1}_{n}^{T}X\\&\hspace{1.5em}
	-\frac{n}{n+1}\left(\bar{\textbf{x}}\bar{\textbf{x}}^{T}+2\bar{\textbf{x}}\textbf{x}_{i}^{T}-n\textbf{x}_{i}\textbf{x}_{i}^{T} \right)\\
	&=nS+n\bar{\textbf{x}}\bar{\textbf{x}}- \frac{n}{n+1}\left(\bar{\textbf{x}}\bar{\textbf{x}}^{T}+2\bar{\textbf{x}}\textbf{x}_{i}^{T}-n\textbf{x}_{i}\textbf{x}_{i}^{T} \right)\\
	&=nS-\frac{n}{n+1}\left(2\bar{\textbf{x}}\textbf{x}_{i}^{T}-\bar{\textbf{x}}\bar{\textbf{x}}^{T}-\textbf{x}_{i}\textbf{x}_{i}^{T} \right)\\
	&=nS+\frac{n}{n+1}\left(\bar{\textbf{x}}-\textbf{x}_{i} \right)\left(\bar{\textbf{x}}-\textbf{x}_{i} \right)^{T}\\
	&=nS+\left(\bar{\textbf{x}}-\textbf{x}_{i} \right)\left(\bar{\textbf{x}}-\textbf{x}_{i} \right)^{T}\\&\hspace{1.5em}-\frac{1}{n+1}\left(\bar{\textbf{x}}-\textbf{x}_{i} \right)\left(\bar{\textbf{x}}-\textbf{x}_{i} \right)^{T}.
	\end{aligned}
	\end{equation*}
\end{proof}
\end{document}